\documentclass{article}

\usepackage[preprint]{ewrl_2025}
\PassOptionsToPackage{round}{natbib}

\usepackage[utf8]{inputenc} 
\usepackage[T1]{fontenc}    
\usepackage{lmodern}

\usepackage{hyperref}       
\usepackage{url}            
\usepackage{booktabs}       
\usepackage{amsfonts}       
\usepackage{nicefrac}       
\usepackage{microtype}      

\usepackage{amsmath,amsthm,amssymb,amsfonts}
\usepackage{algorithm,algorithmic}
\usepackage[dvipsnames,table]{xcolor}
\usepackage{xspace}
\usepackage{dsfont}
\usepackage{makecell}
\usepackage{nicefrac}

\usepackage[disable]{todonotes}

\newtheorem{definition}{Definition}

\newtheorem{corollary}{Corollary}
\newtheorem{theorem}{Theorem}
\newtheorem{lemma}{Lemma}

\newcommand{\dpsprt}{{\textcolor{blue}{\texttt{DP-SPRT}}}}
\newcommand{\outint}{{\textcolor{red}{\texttt{OutsideInterval}}}}

\usepackage{macros}
\usepackage[nohints]{minitoc}

\title{\dpsprt: Differentially Private Sequential Probability Ratio Tests}

\author{Thomas Michel\\Univ. Lille, Inria, CNRS, Centrale Lille,\\ UMR 9189-CRIStAL, France
 \And Debabrota Basu\\Univ. Lille, Inria, CNRS, Centrale Lille,\\ UMR 9189-CRIStAL, France
 \And Emilie Kaufmann\\Univ. Lille, CNRS, Inria, Centrale Lille,\\ UMR 9189-CRIStAL, France}

\begin{document}

\maketitle

\setcounter{parttocdepth}{2}
\doparttoc 
\faketableofcontents 

\begin{abstract}
    We revisit Wald's celebrated Sequential Probability Ratio Test for sequential tests of two simple hypotheses, under privacy constraints. We propose \dpsprt, a wrapper that can be calibrated to achieve desired error probabilities and privacy constraints, addressing a significant gap in previous work. \dpsprt{} relies on a private mechanism that processes a sequence of queries and stops after privately determining when the query results fall outside a predefined interval. This \outint{} mechanism improves upon naive composition of existing techniques like AboveThreshold, achieving a factor-of-2 privacy improvement and thus potentially benefiting other continual monitoring procedures. We prove generic upper bounds on the error and sample complexity of \dpsprt{} that can accommodate various noise distributions based on the practitioner's privacy needs. We exemplify them in two settings: Laplace noise (pure Differential Privacy) and Gaussian noise (Rényi differential privacy). In the former setting, by providing a lower bound on the sample complexity of any $\epsilon$-DP test with prescribed type I and type II errors, we show that \dpsprt{} is near optimal when both errors are small and the two hypotheses are close. Moreover, we conduct an experimental study revealing its good practical performance.
\end{abstract}

\section{Introduction}
\label{sec:intro}

Sequential hypothesis testing offers an elegant solution when data collection is costly or time-consuming, allowing decisions to be reached with fewer observations on average compared to fixed-sample approaches. The Sequential Probability Ratio Test (SPRT), developed by \cite{Wald45SPRT}, is the foundational algorithm in this field, achieving optimal performance for simple hypothesis testing with independent and identically distributed observations.

Sequential testing procedures are evaluated based on their error rates (Type I and Type II errors) and expected sample size (or \emph{sample complexity}) under each hypothesis. \cite{Wald48OptSPRT} proved that SPRT minimizes sample complexities subject to error probability constraints—a remarkably strong optimality result in statistical decision theory.
Despite its theoretical optimality and applications in clinical trials \citep{armitage1950sequential, armitage1954sequential}, quality control, and manufacturing \citep{wald2004sequential, siegmund2013sequential}, SPRT faces a significant modern challenge: \emph{privacy}. When applied to sensitive data, such as patient records in medical trials, financial information, or personal behavior patterns, sequential testing procedures risk exposing confidential information. Privacy concerns are particularly acute for the most recent observations, as the decision to terminate data collection can reveal information about them.

Consider a clinical trial where patients sequentially receive treatments and their responses are recorded. With the classical SPRT, if the test concludes that a drug increases the chance of observing an undesirable effect, then we can deduce the effect actually happened to the last participant in the study. More sophisticated attacks could deduce many other pieces of private information about participants throughout the trial. Even if individual patient responses are kept confidential, the pattern of treatment assignments and the stopping decision itself can leak sensitive information about participants. This privacy risk requires a formal framework for privacy-preserving sequential testing.

\textbf{Related Work.} Differential privacy (DP)~\citep{dwork2006calibrating} has emerged as the standard framework for addressing privacy concerns in statistical analyses. By ensuring that the presence or absence of any single data point cannot significantly affect the outcome, DP offers a principled method to protect individual data while enabling meaningful statistical inference.
Researchers have explored DP hypothesis testing-- primarily in the fixed-sample setting. For example, \cite{gaboardi2016differentially} studies private chi-squared testing, \cite{canonne2019structure} investigates optimal private tests for simple hypotheses, \cite{couch2019differentially} studies nonparametric hypothesis testing with DP, and \cite{kazan2023test} proposes a generic framework for any fixed-sample hypothesis test. These approaches, however, are not directly applicable to sequential testing due to the adaptive nature of sample collection.

For sequential settings, existing differentially private mechanisms like AboveThreshold~\citep{dwork2006calibrating} and the sparse vector technique enable monitoring whether query results exceed specified thresholds while preserving privacy. However, many applications—including sequential hypothesis testing—require monitoring both upper and lower bounds simultaneously, leading to suboptimal privacy guarantees when using naive composition.

\cite{wang2022differentially} explored the privatization of SPRT by adding Laplace noise to thresholds and using the exponential mechanism for the final decision. However, due to their specific use case, their approach provides a weaker notion of privacy that only requires that critical privacy failures do not happen only for a subset of plausible datasets.
In contrast, our work maintains DP throughout the testing process while working to minimize stopping time.
\cite{Cummings22DPSPRT} developed PrivSPRT, a private version of the SPRT using Gaussian noise with Renyi DP guarantees \citep{mironov2017renyi}. The authors provide upper bounds on the error probabilities given fixed thresholds and noise level, but do not specify how to calibrate these thresholds to achieve desired Type I and Type II error rates while maintaining a desired privacy guarantee. In practice, this forces to rely on Monte Carlo simulation to determine appropriate thresholds, a process that may itself leak privacy in certain scenarios \citep{papernot2021hyperparameter}.

Another related line of work is Differentially Private Best Arm Identification (DP-BAI) in bandits~\citep{dpseOrSheffet,kalogerias2020best,azize2023complexity,azize2024differentially}. BAI is a generalization of sequential hypothesis testing, where we simultaneously test means of multiple distributions to identify the largest one.
\citep{azize2023complexity,azize2024differentially} design algorithms for $\varepsilon$-DP-BAI whose sample complexity has irreducible problem-dependent constants compared to the lower bound, keeping the optimal algorithm design open. Thus, we focus on the archetypal setup of testing two hypotheses, and aim to explicate the privacy-utility trade-offs.

\textbf{Contributions.} (1) We introduce a general framework for differentially private SPRTs, called \dpsprt, that supports \textit{various noise mechanisms} while providing privacy guarantees and error control. We propose a generic sample complexity analysis of \dpsprt{} for Bernoulli observations and particularize it to Gaussian noise (for Renyi DP) and Laplace noise (for $\varepsilon$-DP). (2) This framework is based on the novel \outint{} mechanism for differentially privately comparing a sequence of queries to two thresholds. This might be of independent interest for applications to DP problems with or without continual observations. (3) We further prove a lower bound on the sample complexity of any $\varepsilon$-DP algorithm that achieves prescribed type I and type II errors. This lower bound reveals two distinct regimes: one where privacy has minimal impact on sample complexity, and another where higher privacy requirements substantially increase the expected sample size. It also shows that \emph{\dpsprt{} with Laplace noise is near optimal in a regime of small errors and when the two hypotheses are close}. (4) On the practical side, we use subsampling to propose a practical improvement of \dpsprt{} with better privacy-utility trade-off, especially in high privacy regimes.

Our analysis of \dpsprt{} is conducted for Bernoulli observations, which model binary outcomes 
present in applications such as clinical trials, quality control, and A/B testing. Bernoulli distributions form a one-dimensional exponential family with bounded support in $[0,1]$, making them suitable for DP mechanisms requiring bounded sensitivity. This setting provides an illustration of privacy-utility trade-offs while maintaining mathematical tractability.

\section{Preliminaries: Sequential Tests, Differential Privacy, and SPRT}\label{sec:prelim}
\textbf{Sequential hypothesis testing} allows collecting data incrementally and making decisions as soon as sufficient evidence accumulates. Given a stream of i.i.d. observations $X_i$ from some parametric distribution $\nu_{\theta}$ with $\theta \in \Theta$, we consider two simple hypotheses $\cH_0: (\theta = \theta_0)$ and $\cH_1: (\theta =\theta_1)$, where $\theta_0$ and $\theta_1$ are two distinct parameters in $\Theta$. We let $\bP_{\theta}$ be the probability space under which the sequence $(X_i)_{i\in \N}$ is i.i.d. under $\nu_{\theta}$.

A sequential test consists of a \textit{stopping rule} that determines when to stop collecting data, and a \textit{decision rule} that selects one of the hypotheses. Specifically, a sequential test yields a pair $(\tau, \hat{d})$, where $\tau$ is a {stopping time} with respect to the filtration $\mathcal{F}_t = \sigma(X_1,\ldots,X_t)$, and the final decision $\hat{d} \in \{0,1\}$ is $\cF_{\tau}$-measurable. Given prescribed error levels $\alpha, \beta \in (0,1)$, we seek to design tests that have a type I error smaller than $\alpha$ and a type II error smaller than $\beta$. These tests should have a small \textit{sample complexity}, $\bE_{\theta_i}[\tau]$, under either hypotheses.
\begin{definition}[($\alpha,\beta$)-Correct Sequential Test] A sequential test $(\tau,\hat{d})$ is $(\alpha,\beta)$-correct if it satisfies $\bP_{\theta_0}(\hat{d} = 1) \leq \alpha$ and $\bP_{\theta_1}(\hat{d} = 0) \leq \beta$.
\end{definition}


\textbf{Differentially Private Sequential Tests.} Differential Privacy (DP)~\citep{dwork2006calibrating} is usually defined for a randomized mechanism based on a dataset of fixed size. It is expressed as a constraint on the randomness used by the mechanism: the distribution of its output cannot change too much if a single entry in its input (the dataset) is modified. It is possible to extend DP to sequential tests by considering as input an \textit{infinite observation sequence} $\underline{X} \defn (X_1,X_2,\dots)$. The mechanism $\cM(\underline{X}) \defn (\tau(\underline{X}),\hat{d}(\underline{X}))$ outputs the stopping time and decision of the sequential test run on the data stream $\underline{X}$ (regardless of the distribution of that stream). We formalize below for sequential tests different notions of privacy commonly used in the literature.




\begin{definition}[Neighboring Observation Sequences]
    Two observation sequences $\underline{X} \defn (X_1, X_2, \ldots)$ and $\underline{X}' \defn (X'_1, X'_2, \ldots)$ are neighboring if they differ in at most one observation, i.e., if there exists at most one index $i$ such that $X_i \neq X'_i$.
\end{definition}
\begin{definition}[Differentially Private Sequential Test] \label{def:DP}
    A \textit{sequential test} with associated randomized mechanism $\cM(\underline{X}) \defn (\tau(\underline{X}),\hat{d}(\underline{X}))$ satisfies:

    \textbf{(i)} \textbf{$(\varepsilon,\delta)$-DP} if for all neighboring observation sequences $\underline{X}$ and $\underline{X}'$, and for all possible outputs $(t,d) \in \mathbb{N} \times \{0,1\}$,
    \[\Pr[\mathcal{M}(\underline{X}) = (t,d)] \leq e^{\varepsilon}\Pr[\mathcal{M}(\underline{X}') = (t,d)] + \delta\;.\]
    We say that the test satisfies $\epsilon$-pure DP if $\delta = 0$.

    \textbf{(ii)} \textbf{$(\alpha,\epsilon)$-R\'enyi DP} if for all neighboring observation sequences $\underline{X}$ and $\underline{X}'$:
    \[D_{\alpha}(\mathcal{M}(\underline{X}) \| \mathcal{M}(\underline{X}')) \leq \epsilon\]
    where $D_{\alpha}(P\|Q) = \frac{1}{\alpha-1}\log\mathbb{E}_{x \sim Q}\left[\left(\frac{P(x)}{Q(x)}\right)^{\alpha}\right]$ is the R\'enyi divergence of order $\alpha$.
    A test has an \emph{RDP} profile $\tilde{\epsilon}(\alpha)$ if it satisfies $(\alpha,\tilde{\epsilon}(\alpha))$-\emph{RDP} for all $\alpha > 1$.
\end{definition}

R\'enyi Differential Privacy (RDP)~\citep{mironov2017renyi} provides a useful alternative framework for privacy analysis that often leads to tighter composition bounds than standard differential privacy. $(\alpha, \varepsilon')$-RDP can be converted to approximate $(\varepsilon, \delta)$-DP~\citep{dwork2014algorithmic} using standard conversion formula, i.e., for $\varepsilon > \varepsilon'$, $\delta=\exp(-(\alpha-1)(\varepsilon-\varepsilon'))$~\citep{abadi2016deep}.


Definition~\ref{def:DP} captures the fact that the output of a sequential test consists of both when to stop and what decision to make. Thus, we must ensure that a single observation cannot substantially affect either the stopping time or the final decision. Our goal is to develop $(\alpha,\beta)$-correct tests with privacy guarantees, expressed either in DP or in RDP. To do so, it is natural to rely on an optimal (non-private) $(\alpha,\beta)$-correct test-- the SPRT~\citep{Wald48OptSPRT}.

\textbf{SPRT: The Sequential Probability Ratio Test.} Let $f_{\theta}$ denote the density of $\nu_{\theta}$ with respect to some reference measure. Then, SPRT with thresholds $\gamma_0$ and $\gamma_1$ is defined as:

\begin{align*}
    \tau \defn \inf \Bigg\{ t \in \N : \Gamma_t := \prod_{i=1}^{t}\frac{f_{\theta_1}(X_i)}{f_{\theta_0}(X_i)} \notin \left(\gamma_0,\gamma_1\right)\Bigg\}
    \ \text{ and } \ \ \hat{d} = \left\{\begin{array}{cl}
                                            0 & \text{ if } \Gamma_{\tau} \leq \gamma_0 \\
                                            1 & \text{ if } \Gamma_{\tau} \geq \gamma_1
                                        \end{array}\;.
    \right.\end{align*}

What makes SPRT remarkable is its optimality. As shown by \cite{Wald48OptSPRT}, for any test with the same error probabilities as SPRT, the expected sample size cannot be smaller under either hypothesis. 
Still, using SPRT in practice requires calibrating the thresholds $\gamma_0$ and $\gamma_1$ to make it $(\alpha,\beta)$-correct. 
In this work, we rely on an exact calibration\footnote{\cite{Wald45SPRT} rather advocates the use of the heuristic thresholds $\gamma_0 = \frac{\beta}{1-\alpha}$ and $\gamma_1 = \frac{1-\beta}{\alpha}$ but they cannot guarantee both type I and type II errors}, i.e., $\gamma_0 = \beta$ and $\gamma_1 = {1}/{\alpha}$,
for which SPRT can be proved to be $(\alpha,\beta)$-correct using an elegant change-of-measure argument (see Appendix~\ref{appendix:correctness-proofs}).
%


For one-parameter canonical exponential families~\citep{brown1986fundamentals}, in which $\Theta \subseteq \R$ and the density can be written $f_{\theta}(x) = \exp(\theta x - b(\theta))$ where $b : \Theta \rightarrow \R$ is called the log-partition function, SPRT takes a more explicit form (see Appendix~\ref{appendix:sprt-exponential-families}). To fix the ideas, we assume in sequel that $\theta_0 < \theta_1$. Letting $\bar{X}_n = \sum_{i=1}^n X_i/n$ denote the empirical mean, $\mu_i$ be the mean of the distribution $\nu_{\theta_i}$, and $\KL$ the Kullback-Leibler divergence, we have $\tau \defn \min(\tau_0 , \tau_1)$, where
\begin{align}  \label{eq:sprt}
    \begin{split}
        \tau_0 & \defn \inf\bigg\{n \in \N_{>0}: \bar{X}_n \leq  \ell_0^n \bigg\}, \\
        \tau_1 & \defn \inf\bigg\{n \in \N_{>0}: \bar{X}_n \geq \ell_1^n \bigg\}
    \end{split}
\end{align}
with thresholds:
\begin{align*}
    \ell_0^n & = \mu_0 + \frac{\KL(\nu_{\theta_0}, \nu_{\theta_1}) - \log(1/\beta)/n}{\theta_1-\theta_0},  \\
    \ell_1^n & = \mu_1 - \frac{\KL(\nu_{\theta_1}, \nu_{\theta_0}) - \log(1/\alpha)/n}{\theta_1-\theta_0}.
\end{align*}
and $\hat{d} = i$ if $\tau = \tau_i$. This class of distributions encompasses several well-known distributions such as Bernoulli, Poisson or Gaussian distributions with known variance. We propose in Section~\ref{sec:dpsprt} a private version of this test, and analyze it for Bernoulli distributions. SPRT relies on the empirical mean, and private versions of the empirical mean are \textit{mostly} available under a bounded sensitivity assumption\footnote{The sensitivity of a function $f: \mathcal{X}^n \to \mathbb{R}$ is the maximum change in the function's value when a single observation is modified:
    $\Delta_f = \max_{\underline{X}, \underline{X}' \text{ neighboring}} |f(\underline{X}) - f(\underline{X}')|.$}, which holds for bounded distributions only. Note that, we have approaches that bound the sensitivity of subGaussian variables with high probability. But it introduces a sample-dependent $\sqrt{\log n}$ term in the sensitivity and loss in the probability of correctness~\citep{kamath2022private}. To avoid this nuance and focus on explicating tight privacy-utility trade-off, we focus on the Bernoulli distributions.



Equation~\eqref{eq:sprt} reveals a clear structure of SPRT. We have two thresholds that determine when to stop collecting data-- one for accepting $\cH_0$ and one for accepting $\cH_1$. When adapting this test to ensure DP, a na\"ive approach might compose two instances of the AboveThreshold mechanism~\citep{kaplan2021sparse}, one for each decision boundary. Instead, we present a tighter mechanism that jointly handles both thresholds with a smaller privacy cost.

\section{\outint{}  Mechanism}\label{sec:outside_interval}
Now, we introduce a novel private mechanism that forms the foundation of our approach to differentially private sequential testing. Let $f_i : D \rightarrow \R$ be a sequence of queries with sensitivity $\Delta$ ($D$ is either a finite dataset or an observation sequence), and $(T_0^{i},T_1^{i}) \in \R^2$ with $T_0^i \leq T_1^{i}$. The \outint{} mechanism takes $D$ as input and outputs a sequence $\underline{a}$ of the form $(\bot,\dots,\bot,\top_{0})$ or $(\bot,\dots,\bot,\top_{1})$ where $a_i = \bot$ indicates that $f_i(D)$ should belong to $[T_0^{i},T_1^{i}]$ and $a_i = \top_{0}$ (resp. $a_i = \top_{1}$) indicates that $f_i(D)$ should be smaller than $T_0^{i}$ (resp. larger than $T_1^{i}$). The length of $\underline{a}$ is the first query falling outside of the interval.
This design is directly motivated by SPRT, which monitors whether a test statistic falls outside an interval defined by two thresholds, rather than crossing a single threshold.

\outint{} (Algorithm~\ref{alg:outside_interval}) proceeds by adding noise to both the query results and the thresholds, and returns either $\top_0$, $\top_1$, or $\bot$ depending on whether the noisy query result is below the lower threshold, above the upper threshold, or within the interval. We emphasize that the algorithm adds the same noise realizations to both threshold comparisons, which allows for tighter privacy guarantees.

This mechanism extends AboveThreshold \citep{dwork2006calibrating} algorithm to simultaneously handle two thresholds. Unlike the approach of~\cite{Cummings22DPSPRT}, which composes two separate instances of AboveThreshold for the upper and lower thresholds, \outint{} handles both thresholds simultaneously, providing better privacy guarantees.

\begin{algorithm}[t]
    \caption{\outint{}}\label{alg:outside_interval}
    \begin{algorithmic}[1]
        \REQUIRE Private database $D$, adaptively chosen stream of queries $f_1,\ldots$ where each $f_i$ has sensitivity $\Delta$, sequences of thresholds $(T_0^i)_{i=1}^\infty, (T_1^i)_{i=1}^\infty$, noise-distributions $\cD_{Z}, \cD_Y$

        \STATE $Z \gets \text{Sample from } \cD_{Z}$

        \FOR{Each query $i$}
        \STATE $Y_i \gets \text{Sample from } \cD_Y$
        \IF{$f_i(D) + Y_i \leq T_0^i - Z$}
        \STATE Halt and output $a_i = \top_0$
        \ELSIF{$f_i(D) + Y_i \geq T_1^i + Z$}
        \STATE Halt and output $a_i = \top_1$
        \ELSE
        \STATE Output $a_i = \bot$
        \ENDIF
        \ENDFOR
    \end{algorithmic}
\end{algorithm}

We express the privacy of \outint{} in terms of the DP and RDP guarantees of the noise-adding mechanisms associated with $Y_i$ and $Z$. This allows us to provide a flexible analysis that can accommodate different noise distributions.

\begin{definition}[Noise-adding Mechanism]
    Given a continuous noise distribution $\mathcal{D}$, a noise-adding mechanism $\mathcal{M}$ is a randomized algorithm that takes as input a dataset $D$ and a query $f:D\to \bR$, and outputs a noisy answer $\mathcal{M}(D,f) = f(D) + X$, where $X$ is drawn from $\mathcal{D}$.
\end{definition}

\begin{theorem}
    \label{thm:general-outside-interval}
    Let $\tau$ be the (random) length of the sequence output by \outint{}.

    (i) If the noise-adding mechanisms corresponding to $\cD_{Z}$ and $\cD_Y$ satisfy $\epsilon_Z$-DP for queries with sensitivity $\Delta$, and $\epsilon_Y$-DP for queries with sensitivity $2\Delta$, then \outint{} is $(\epsilon_Z + \epsilon_Y)$-DP.

    (ii) If the noise-adding mechanism corresponding to $\cD_{Z}$ has an RDP profile $\epsilon_{Z}(\alpha)$ for queries with sensitivity $\Delta$ and the one for $\cD_Y$ has an RDP profile $\epsilon_Y(\alpha)$ for queries with sensitivity $2\Delta$, then \outint{} (denoted by $\mathcal{A}$) satisfies: for all $\alpha>1$,
    {
            \begin{align*}
                D_{\alpha}(\mathcal{A}(D)\|\mathcal{A}(D')) & \leq \frac{\alpha - 1/2}{\alpha - 1}\epsilon_{Z}(2\alpha) + \epsilon_{Y}(\alpha) + \frac{\log \left(2\mathbb{E}_{z \sim \cD_{Z}} [\mathbb{E}_{(\tau,\hat d)\sim\cA(D')}[\tau|Z=z]^{2}]\right)}{2(\alpha - 1)}.
            \end{align*}}
\end{theorem}
This result shows that the privacy of our mechanism depends on both the threshold noise $Z$ and the query noise $Y$. For $\epsilon$-DP, we achieve $(\epsilon_Z + \epsilon_Y)$-DP, which improves by a factor of 2 over using two separate instances of AboveThreshold. For RDP, the result has a similar flavor as in \citep{Cummings22DPSPRT}, but our specific analysis also reduces each of the terms by a factor of 2 compared to their. The full proof is provided in Appendix~\ref{appendix:privacy-proofs}.

This privacy analysis forms the basis for our differentially private sequential tests in the next section. Unlike prior work by \cite{Cummings22DPSPRT}, which only addressed Rényi DP for Gaussian noise, our approach supports both pure $\epsilon$-DP and RDP guarantees. Beyond hypothesis testing, \outint{} can serve as a drop-in replacement for applications that use AboveThreshold where we want to ensure privacy under continual observations~\citep{dwork2010differential}, allowing monitoring of two thresholds simultaneously and reporting which one was crossed without additional privacy cost, contrary to the approach that would compose two separate instances of AboveThreshold.



\section{Differentially Private Sequential Probability Ratio Test}
\label{sec:dpsprt}
Building on the threshold mechanism presented in Section~\ref{sec:outside_interval}, we now introduce a general framework for differentially private sequential probability ratio tests (\dpsprt{}) that supports various noise distributions while maintaining error control and privacy guarantees.

\subsection{General Definition}

We consider the sequential hypothesis test between $\cH_0: (\theta = \theta_0)$ and $\cH_1: (\theta = \theta_1)$ (with $\theta_0 < \theta_1$) based on i.i.d. samples $X_1, X_2, \ldots$ from $\nu_{\theta}$, that belongs to a one-dimensional exponential family with natural parameter $\theta$. We seek to design a sequential test $(\tau,\hat{d})$ that (i) satisfies differential privacy, (ii) controls type I and type II error probabilities at specified levels $\alpha$ and $\beta$, and (iii) minimizes the expected sample size.

The idea of our method is to privatize the classical SPRT by adding carefully calibrated noise to both the thresholds and test statistic. \dpsprt{} requires two noise distributions $\cD_Y$ and $\cD_Z$, a correction function $C(n,\delta)$ and a parameter $\gamma \in (0,1)$ allocating the error probability between the SPRT error and the privacy noise. It is defined as $\tau \defn \min(\tau_0, \tau_1)$ with
\begin{align*}
    \tau_0 & \defn \inf\bigg\{n \in \mathbb{N}_{>0}: \bar{X}_n + \frac{Y_n}{n} \leq \tilde{\ell}_0^n - \frac{Z}{n} \bigg\}, \\
    \tau_1 & \defn \inf\bigg\{n \in \mathbb{N}_{>0}: \bar{X}_n + \frac{Y_n}{n} \geq \tilde{\ell}_1^n + \frac{Z}{n} \bigg\}
\end{align*}
where the private thresholds are:
\begin{align*}
    \tilde{\ell}_0^n & = \mu_0 + \frac{\KL(\nu_{\theta_0}, \nu_{\theta_1}) - \frac{\log(1/(\gamma\beta))}{n}}{\theta_1-\theta_0} - C(n, (1-\gamma)\beta),   \\
    \tilde{\ell}_1^n & = \mu_1 - \frac{\KL(\nu_{\theta_1}, \nu_{\theta_0}) - \frac{\log(1/(\gamma\alpha))}{n}}{\theta_1-\theta_0} + C(n, (1-\gamma)\alpha).
\end{align*}
Here, $Y_n \overset{i.i.d.}{\sim} \cD_Y$ are the successive query noise and $Z \sim \cD_Z$ is the threshold noise. The decision rule is defined as $\hat{d} = i$ if $\tau = \tau_i$. \dpsprt{} is described formally in Algorithm \ref{alg:dpsprt}, where the blue terms highlight the differences with the vanilla SPRT.


\begin{algorithm}[t]
    \caption{SPRT and  \textcolor{blue!70}{\dpsprt{}}}
    \label{alg:dpsprt}
    \begin{algorithmic}[1]
        \REQUIRE Parameters $\theta_0$, $\theta_1$, error probabilities $\alpha$, $\beta$
        \REQUIRE \textcolor{blue!70}{
            Noise distributions $\mathcal{D}_Z$, $\mathcal{D}_Y$, error allocation $\gamma \in (0,1)$, correction function $C$}

        \STATE Initialize $n = 0$
        \STATE \textcolor{blue!70}{Sample threshold noise $Z \sim \mathcal{D}_{Z}$}

        \WHILE{true}
        \STATE $n \gets n + 1$
        \STATE Observe sample $X_n$
        \STATE Compute $\bar{X}_n = \frac{1}{n}\sum_{i=1}^n X_i$
        \STATE \textcolor{blue!70}{Sample query noise $Y_n \sim \mathcal{D}_Y$}

        \STATE $\hat{T}_0^n\! \gets \mu_0 + \frac{\KL(\nu_{\theta_0}, \nu_{\theta_1}) + \log(\textcolor{blue!70}{\gamma}\beta)/n}{\theta_1-\theta_0} \textcolor{blue!70}{- C\left(n,(1-\gamma)\beta\right)}$
        \STATE $\hat{T}_1^n\! \gets \mu_1 - \frac{\KL(\nu_{\theta_1}, \nu_{\theta_0}) + \log(\textcolor{blue!70}{\gamma}\alpha)/n}{\theta_1-\theta_0} \textcolor{blue!70}{+ C\left(n,(1-\gamma)\alpha\right)}$

        \IF{$\bar{X}_n \textcolor{blue!70}{+ \frac{Y_n}{n}} \leq \hat{T}_0^n \textcolor{blue!70}{- \frac{Z}{n}}$}
        \STATE Halt and accept $\cH_0$ ($\hat{d} = 0$)
        \ELSIF{$\bar{X}_n \textcolor{blue!70}{+ \frac{Y_n}{n}}  \geq \hat{T}_1^n \textcolor{blue!70}{+ \frac{Z}{n}}$}
        \STATE Halt and accept $\cH_1$ ($\hat{d} = 1$)
        \ENDIF
        \ENDWHILE
    \end{algorithmic}
\end{algorithm}

\subsection{Generic analysis}

We now propose a generic analysis of \dpsprt{}, which depends on properties of the noise distributions $\cD_Y$ and $\cD_Z$ as well as the correction function.

\textbf{$\bm{(i)}$ Privacy.}
The privacy guarantees follow from Theorem \ref{thm:general-outside-interval} by noting that \dpsprt{} can be seen as an instance of \outint{}. Specifically, the sensitivity of empirical means of Bernoulli observations is ${n}^{-1}$. By multiplying both sides of the stopping conditions by $n$, we can express \dpsprt{} in terms of the \outint{} mechanism with queries $f_n(D) = \sum_{i=1}^n X_i$, whose sensitivity is $\Delta = 1$, noise $Y_n$, and appropriately defined thresholds.

\begin{theorem}[Privacy]
    \label{thm:dpsprt-privacy}
    Let $\mathcal{A}$ be the \dpsprt{} algorithm and $\tau$ the (random) stopping time.

    (i) \textbf{$\bm\varepsilon$-DP:} If the noise-adding mechanisms corresponding to distributions $\mathcal{D}_{Z}$ and $\mathcal{D}_Y$ satisfy $\epsilon_Z$-DP for queries with sensitivity $1$, and $\epsilon_Y$-DP for queries with sensitivity $2$, respectively, then \dpsprt{} satisfies $(\epsilon_{Z}+\epsilon_Y)$-Differential Privacy.

    (ii) \textbf{$\bm{(\alpha, \varepsilon)}$-RDP:} If the noise-adding mechanism corresponding to distribution $\mathcal{D}_{Z}$ and $\mathcal{D}_{Y}$ have an RDP profile $\epsilon_{Z}(\alpha)$ for queries with sensitivity $1$ and  $\epsilon_{Y}(\alpha)$ for queries with sensitivity $2$, then \dpsprt{} satisfies: for all $\alpha>1$,
    {
            \begin{align}
                \mathbb{D}_{\alpha}(\mathcal{A}(D)\|\mathcal{A}(D')) & \leq \frac{\alpha - 1/2}{\alpha - 1}\epsilon_{Z}(2\alpha) + \epsilon_{Y}(\alpha)  + \frac{\log \left(2\mathbb{E}_{z \sim \mathcal{D}_{Z}} [\mathbb{E}_{(\tau,\hat d)\sim\cA(D')}[\tau|z]^{2}]\right)}{2(\alpha - 1)}.\label{eq:RDP_upper_bound}
            \end{align}}
\end{theorem}

\textbf{$\bm{(ii)}$ Correctness.} We establish correctness of \dpsprt{} with the following theorem.

\begin{theorem}[Correctness of \dpsprt{}]
    \label{thm:dpsprt-correctness}
    For any error allocation $\gamma \in (0,1)$, \dpsprt{} is $(\alpha,\beta)$-correct if the noises and correction function satisfy:
    \begin{align}
        \forall \delta \in (0,1), & \   \sum_{n=1}^{\infty} \mathbb{P}\left(\frac{Y_n}{n} - \frac{Z}{n} > C\left(n, \delta\right)\right)   \leq \delta\;,\label{ass:correction}  \\
        \forall \delta \in (0,1), & \   \sum_{n=1}^{\infty} \mathbb{P}\left(\frac{Y_n}{n} + \frac{Z}{n} < -C\left(n, \delta\right)\right)  \leq \delta\;.\label{ass:correction2}
    \end{align}
    In particular, if the distribution $\mathcal{D}_Y$ is symmetric around zero, we remark that condition \eqref{ass:correction} alone is sufficient, as it implies condition \eqref{ass:correction2}.
\end{theorem}

The main idea is to decompose the error probabilities using a union bound into the error from the standard SPRT (bounded via a change-of-measure argument) and the error from privacy noise (controlled through the correction function). The parameter $\gamma$ allocates the total error budget between these two sources. The symmetry condition ensures that both tail bounds needed for Type I and Type II error control are satisfied simultaneously. See Appendix~\ref{appendix:correctness-proofs} for the full proof.

\textbf{$\bm{(iii)}$ Sample Complexity.}
The following theorem provides a generic upper bound on the expected sample complexity for any noise distribution satisfying appropriate conditions.

\begin{theorem}[Sample Complexity of \dpsprt{}]
    \label{thm:sc_general}
    Let us assume conditions \eqref{ass:correction} and \eqref{ass:correction2} hold. For any error allocation $\gamma \in (0,1)$, the expected stopping time of \dpsprt{} satisfies:
    {
    \begin{align*}
        \bE_{\theta_0}[\tau] & \leq 1 + (1-\gamma)\beta + \frac{1}{1-e^{-\frac{(\TV(\nu_{\theta_0}, \nu_{\theta_1}))^4}{2(\theta_1 - \theta_0)^2}}} + N_0, \\
        \bE_{\theta_1}[\tau] & \leq 1 + (1-\gamma)\alpha + \frac{1}{1-e^{-\frac{(\TV(\nu_{\theta_0}, \nu_{\theta_1}))^4}{2(\theta_1 - \theta_0)^2}}} + N_1
    \end{align*}}
    where $N_0 = N(\theta_0,\theta_1,\beta,\gamma)$, $N_1 = N(\theta_1,\theta_0,\alpha,\gamma)$, and
        {
            \begin{align*}
                N(\theta,\theta',\delta,\gamma)  = \inf \bigg\{n : \frac{\log(1/(\delta\gamma))}{n|\theta-\theta'|} + 2C(n,(1-\gamma)\delta)   \leq \frac{\KL(\nu_{\theta}, \nu_{\theta'})}{2|\theta - \theta'|}\bigg\}.
            \end{align*}}
\end{theorem}

This result shows how the required number of samples depends on the divergence between the distributions, the desired error rates, and the noise added for privacy. Specific upper bounds on $N$ can be derived for different noise mechanisms. Note that this bound leverages Hoeffding's inequality \citep{Boucheronal13CI}, which applies to bounded random variables like Bernoulli distributions. For other exponential families, one could replace it by a Chernoff-type inequality, resulting in more complex expressions than the form presented here, which is specific to our Bernoulli setting. See Appendix~\ref{appendix:sample-complexity-proofs} for the proof and further analysis.

\subsection{Variants of \dpsprt{}}
The generic \dpsprt{} framework provides flexibility in selecting noise distributions based on privacy requirements. Table~\ref{tab:noise_implementations} summarizes two key implementations: Laplace noise for pure $\epsilon$-DP and Gaussian noise for Rényi DP.

\begin{table}[ht]
    \centering
    \caption{Comparison of \dpsprt{} instantiations}
    \label{tab:noise_implementations}
    \begin{tabular}{@{}p{2.2cm}p{2.8cm}p{2.8cm}@{}}
        \toprule
                        & \textbf{Laplace}                              & \textbf{Gaussian}                                     \\
        \midrule
        Query noise     & $Y_n \sim \text{Lap}(4/\epsilon)$             & $Y_n \sim \mathcal{N}(0, \sigma_Y^2)$                 \\
        Threshold noise & $Z \sim \text{Lap}(2/\epsilon)$               & $Z \sim \mathcal{N}(0, \sigma_Z^2)$                   \\
        \midrule
        Correction      & $\frac{6\log(n^s\zeta(s)/\delta)}{n\epsilon}$ & $\frac{\sqrt{2\sigma^2\log(n^s\zeta(s)/2\delta)}}{n}$ \\
        \midrule
        Privacy         & $\epsilon$-DP                                 & $(\alpha, \epsilon(\alpha))$-RDP                      \\
        \bottomrule
    \end{tabular}
\end{table}

Here, $\sigma^2 = \sigma_Y^2 + \sigma_Z^2$ and $\epsilon(\alpha) = \frac{\alpha - 1/2}{\alpha - 1} \cdot \frac{\alpha}{\sigma_Z^2} + \frac{\alpha}{2\sigma_Y^2} + \frac{\log(2\mathbb{E}[\mathbb{E}[\tau|Z]^2])}{2(\alpha - 1)}$.

Both variants satisfy conditions \eqref{ass:correction} and \eqref{ass:correction2}, ensuring correctness with the specified error probabilities. The Laplace noise variant provides the strongest privacy guarantee (pure $\epsilon$-DP), while the Gaussian implementation offers Rényi differential privacy that can be converted to approximate $(\epsilon,\delta)$-DP using standard techniques \citep{mironov2017renyi}. Detailed sample complexity analyses for these implementations are presented in Appendix~\ref{appendix:sample-complexity-proofs}.

\section{On \dpsprt{} with Laplace Noise}
\label{sec:laplace}

While using \dpsprt{} with Gaussian noise provides RDP guarantees, some applications require the stronger guarantee of pure $\epsilon$-DP. By examining a lower bound on the sample complexity, we argue that our instance of \dpsprt{} with Laplace noise is actually near-optimal in some regimes. We then propose further enhancements using sub-sampling.
\begin{corollary}[Sample Complexity of \dpsprt{} with Laplace Noise]
    \label{thm:sc_laplace}
    For \dpsprt{} with Laplace noise, using the parameters from Table \ref{tab:noise_implementations}, the function $N$ in Theorem \ref{thm:sc_general} satisfies:
    {
    \begin{align*}
        N(\theta_0,\theta_1,\beta,\gamma) & \leq \frac{24(\theta_1 - \theta_0)\log(\zeta(s)/(1-\gamma)\beta)}{\varepsilon \KL(\nu_{\theta_0}, \nu_{\theta_1})}  + \frac{2\log(1/(\gamma\beta))}{\KL(\nu_{\theta_0}, \nu_{\theta_1})} + o_{\beta \rightarrow 0}\Big(\log(1/\beta)\Big).
    \end{align*}}
\end{corollary}
The bound has two components: one reflecting the statistical difficulty of hypothesis testing (as in non-private settings) and another quantifying the cost of ensuring differential privacy, with $\gamma$ controlling their trade-off. The full derivation appears in Appendix~\ref{appendix:sample-complexity-proofs}.

%
%
%
%

\subsection{Comparison with a Lower Bound}

In Appendix~\ref{appendix:lower-bound-proofs}, we prove the following lower bound on the sample complexity of any $(\alpha,\beta)$-correct test that is $\epsilon$-DP. This lower bound is valid for any distributions $\nu_{\theta}$, and follows from techniques used to prove lower bounds under DP constraints in the multi-armed bandit literature (e.g.,  \cite{azize2023complexity,azize2024privacy}).

\begin{theorem}[Lower Bound for Private Sequential Tests]
    \label{thm:lower-bound}
    Let $(\tau, \hat{d})$ be an $\epsilon$-differentially private test with $\mathbb{P}_{\theta_0}(\hat{d} = 1) \leq \alpha$ and $\mathbb{P}_{\theta_1}(\hat{d} = 0) \leq \beta$. Then:
    \begin{align*}
        \mathbb{E}_{\theta_0}[\tau] & \geq \frac{\kl(\alpha,1-\beta)}{\min\left(\KL(\nu_{\theta_0},\nu_{\theta_1}), \epsilon \cdot \TV(\nu_{\theta_0},\nu_{\theta_1})\right)}  \text{ and } \\ \mathbb{E}_{\theta_1}[\tau] &\geq \frac{\kl(\beta,1-\alpha)}{\min\left(\KL(\nu_{\theta_1},\nu_{\theta_0}), \epsilon \cdot \TV(\nu_{\theta_0},\nu_{\theta_1})\right)}
    \end{align*}
    where $\TV$ is the total variation distance between distributions and $\kl(x,y)\defn \KL(\cB(x),\cB(y)) = x\log(x/y) + (1-x)\log((1-x)/(1-y))$ is the binary relative entropy.
\end{theorem}



We now compare this lower bound to our upper bound for Laplace \dpsprt{}, in the asymptotic regime in which $\alpha$ and $\beta$ go to zero, for which we have $\kl(\alpha,1-\beta) \simeq \log(1/\beta)$ and $\kl(\beta,1-\alpha) \simeq \log(1/\alpha)$. Combining Theorem~\ref{thm:sc_general} with Corollary \ref{thm:sc_laplace}, we have
\begin{align*}
    \limsup_{\beta \rightarrow 0} \frac{\bE_{\theta_0}\left[\tau\right]}{\log(1/\beta)} \leq \frac{48}{\KL(\nu_{\theta_0},\nu_{\theta_1})}\max\left(1, \frac{\theta_1 - \theta_0}{ \varepsilon}\right)\,,
\end{align*}
where $\nu_{\theta_i}$ is the Bernoulli distribution with natural parameter $\theta_i$ and mean $p_i = \tfrac{e^{\theta_i}}{1+e^{\theta_i}}$. Using the simple expression of the TV for Bernoulli distributions in Theorem~\ref{thm:lower-bound} and $p_1 > p_0$,
\begin{align*}
    \liminf_{\beta \rightarrow 0} \frac{\bE_{\theta_0}\left[\tau\right]}{\log(1/\beta)} \geq \max\left(\frac{1}{\kl(p_0, p_1)}, \frac{1}{(p_1 - p_0) \varepsilon}\right).
\end{align*}
By definition, we have $\kl(p_0, p_1) = \KL(\nu_{\theta_0}, \nu_{\theta_1})$. As for the second term, using the expression of KL divergence for exponential families (see Eq. \eqref{def:KL_exponentials} in Appendix~\ref{appendix:sprt-exponential-families}) we have ${\KL(\nu_{\theta_0}, \nu_{\theta_1}) }/{(\theta_1 - \theta_0)} = (b(\theta_1) - b(\theta_0))/(\theta_1 - \theta_0) - p_0$ where $b$ is the log-partition function, which goes to $b'(\theta_1) - p_0 = p_1 - p_0$ when $p_1$ goes to $p_0$. Hence, when $p_1$ and $p_0$ are close, the upper and lower bounds are matching, up to a multiplicative constant. By symmetry, we can get similar nearly matching upper and lower bounds on $\bE_{\theta_1}[\tau]/\log(1/\alpha)$. {This makes $\dpsprt{}$ with Laplace noise near-optimal in the regime of small errors and close hypotheses.}

\subsection{Privacy Amplification by Subsampling}

\dpsprt{} with Laplace noise can be further enhanced using subsampling \citep{steinke2022composition}. That is, in each round the empirical mean is computed based on a random subset of the available observations instead of all observations. This approach allows adding less noise for the same privacy guarantee, and in turn, a smaller correction factor. This is particularly effective in high privacy regimes as demonstrated in the experiments. Details of the implementation of this approach are provided in Appendix~\ref{appendix:subsampling}.

\section{Experimental Analysis}\label{sec:experiments}
\subsection{Experimental Setup}
We conduct experiments on simulated Bernoulli data to evaluate the practical performance of our \dpsprt{} methods. We implement the classical SPRT, PrivSPRT~\citep{Cummings22DPSPRT}, and three variants of our approach: \dpsprt{} with Laplace noise, with Gaussian noise, and with Laplace noise and subsampling (Laplace-Sub). For all methods, we set target error rates $\alpha=\beta=0.05$ and perform 1000 independent trials to estimate the error probabilities and the sample complexity.  We present results for Bernoulli distributions with parameter $p_0 = 0.3$ and $p_1 = 0.7$.

\begin{figure*}[t!]
    \centering
    \begin{tabular}{cc}
        \includegraphics[width=0.49\textwidth]{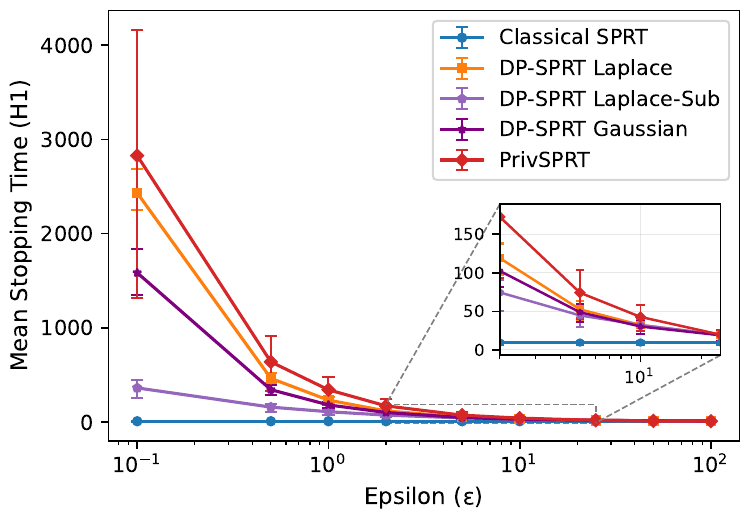} &
        \includegraphics[width=0.49\textwidth]{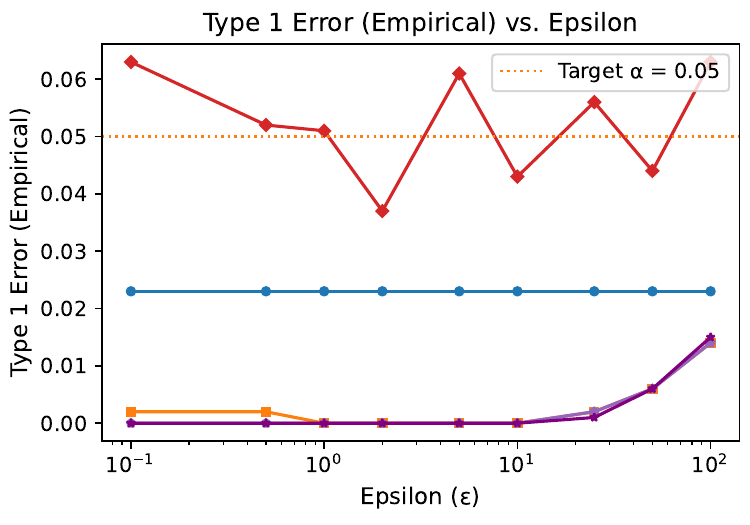}                                                   \\
        (a) Sample size vs. privacy parameter                                                                                      & (b) Type I error vs. privacy parameter
    \end{tabular}
    \caption{Experimental results for different \dpsprt{} variants: (a) Average sample size as privacy parameter varies for $\alpha=\beta=0.05$; (b) Empirical Type I error probability vs. privacy parameter. Error bars represent 95\% percentile intervals over 1000 trials.}\label{fig:main_results}
\end{figure*}

We use parameters of Table~\ref{tab:noise_implementations} for \dpsprt{} with Laplace noise to ensure $\varepsilon$-DP. For  \dpsprt{} with Gaussian noise, following \cite{Cummings22DPSPRT}, we set the variance parameters as $\sigma_Y^2 = 32\ln(1.25/\delta)/\epsilon^2$ and $\sigma_Z^2 = 8\ln(1.25/\delta)/\epsilon^2$, ensuring $(\epsilon/2,\delta)$-DP for queries of sensitivity $2$ and $1$ respectively. 
To compare with an instance of PrivSPRT having a similar RDP guarantee, we set $\sigma_1 = 2\sqrt{2}A\sigma_Z$ and $\sigma_2 = 2\sqrt{2}A\sigma_Y$ where $2A$ is the sensitivity due to their truncation parameter. $A$ is set to 1. The rescaling accounts for the fact that their RDP bound has an additional factor $2$ in the first two terms of Eq.~\eqref{eq:RDP_upper_bound} compared to ours.


The error allocation parameter of \dpsprt{} is set to $\gamma(\epsilon)= \min(1/2, 1-1/\epsilon)$, allowing to recover the classical SPRT as $\epsilon$ increases. For the subsampling version, we adaptively set the subsampling rate to $\min(1,\sqrt{\epsilon/10})$, depending on the privacy parameter. 
More discussion on alternative choices of subsampling rate is in Appendix~\ref{appendix:subsampling}.
We calibrate the thresholds of PrivSPRT through grid search with Monte Carlo simulation (100 runs), while our methods use the theoretical calibration derived in Section~\ref{sec:dpsprt}.


\subsection{Results}

(1) \textbf{Smaller Stopping Time across $\varepsilon$.} Figure~\ref{fig:main_results}(a) shows that as privacy requirements become stringent (smaller $\epsilon$), all private methods require more samples. Our \dpsprt{} variants consistently outperform or match PrivSPRT at equivalent privacy levels, with subsampling providing substantial benefits in high-privacy regimes. This contradicts prior beliefs that Laplace noise is unsuitable for sequential testing \citep{Cummings22DPSPRT}.

(2) \textbf{Smaller Probability of Error across $\varepsilon$.} Figure~\ref{fig:main_results}(b) demonstrates that our \dpsprt{} variants maintain error probabilities ({significantly}) below target across all privacy levels. In contrast, PrivSPRT occasionally exceeds the target rate, likely due to the noise introduced during its Monte Carlo calibration process. The fact that the empirical type I error of \dpsprt{} is much below the target $\alpha$ suggests that our proposed correction function $C(n,\delta)$ is quite conservative. 
In Appendix~\ref{appendix:complementary-experiments}, we provide a more detailed experimental analysis with different $p_0$ and $p_1$, the variance of the stopping time, its dependence on the privacy parameters and the target error probabilities, and a tuning of the correction function.

In brief, \textit{our \dpsprt{} methods achieve better efficiency and accuracy than PrivSPRT while offering theoretically-calibrated thresholds guaranteeing error control without empirical tuning}.



\section{Discussion and Future Work}\label{sec:conclusion}

We introduced \dpsprt{}, an adaptation of the classical SPRT to incorporate various types of privacy protections in sequential tests of simple hypotheses. This general framework supports different noise mechanisms, such as Laplace noise for $\epsilon$-differential privacy and Gaussian noise for Rényi differential privacy. Its design is based on a specialized threshold mechanism, \outint, that concurrently handles both decision boundaries instead of treating them separately. 
Notably, \dpsprt{} with Laplace noise is the first $\varepsilon$-DP sequential test that can be calibrated to attain prescribed type I and type II errors without any hyper-parameter tuning. Moreover, for Bernoullis, its sample complexity is close to the lower bound that we have derived for the minimal sample complexity of any $\varepsilon$-DP test. 
In future work, we will seek to strengthen these optimality guarantees beyond Bernoulli, and extend our findings to composite hypothesis testing, for which  \outint{} could also be useful.

%
%
%
%
%

\section*{Acknowledgements}
This work has been partially supported by the French National Research Agency (ANR) in the framework of the PEPR IA FOUNDRY project (ANR-23-PEIA-0003). We also acknowledge the Inria-ISI, Kolkata Associate Team “SeRAI”, and the ANR JCJC for the REPUBLIC project (ANR-22-CE23-0003-01) for partially supporting the project.

\bibliography{biblio}

\begin{thebibliography}{}

\bibitem[Abadi et~al., 2016]{abadi2016deep}
Abadi, M., Chu, A., Goodfellow, I., McMahan, H.~B., Mironov, I., Talwar, K., and Zhang, L. (2016).
\newblock Deep learning with differential privacy.
\newblock In {\em Proceedings of the 2016 ACM SIGSAC conference on computer and communications security}, pages 308--318.

\bibitem[Armitage, 1950]{armitage1950sequential}
Armitage, P. (1950).
\newblock Sequential analysis with more than two alternative hypotheses, and its relation to discriminant function analysis.
\newblock {\em Journal of the Royal Statistical Society. Series B (Methodological)}, 12(1):137--144.

\bibitem[Armitage, 1954]{armitage1954sequential}
Armitage, P. (1954).
\newblock Sequential tests in prophylactic and therapeutic trials.
\newblock {\em The Quarterly journal of medicine}, 23(91):255--274.

\bibitem[Azize, 2024]{azize2024privacy}
Azize, A. (2024).
\newblock {\em Privacy-Utility Trade-offs in Sequential Decision-Making under Uncertainty}.
\newblock PhD thesis, Universit{\'e} de Lille.

\bibitem[Azize et~al., 2023]{azize2023complexity}
Azize, A., Jourdan, M., Al~Marjani, A., and Basu, D. (2023).
\newblock On the complexity of differentially private best-arm identification with fixed confidence.
\newblock {\em Advances in Neural Information Processing Systems}, 36:71150--71194.

\bibitem[Azize et~al., 2024]{azize2024differentially}
Azize, A., Jourdan, M., Marjani, A.~A., and Basu, D. (2024).
\newblock Differentially private best-arm identification.
\newblock {\em arXiv preprint arXiv:2406.06408}.

\bibitem[Boucheron et~al., 2013]{Boucheronal13CI}
Boucheron, S., S., Lugosi, G., and Massart, P. (2013).
\newblock {\em {Concentration inequalities. A non asymptotic theory of independence.}}
\newblock Oxford University Press.

\bibitem[Brown, 1986]{brown1986fundamentals}
Brown, L.~D. (1986).
\newblock Fundamentals of statistical exponential families with applications in staistical decision theory.
\newblock {\em Instiute of Mathematical Statistics, Lecture Notes-Monograph Series}.

\bibitem[Canonne et~al., 2019]{canonne2019structure}
Canonne, C.~L., Kamath, G., McMillan, A., Smith, A., and Ullman, J. (2019).
\newblock The structure of optimal private tests for simple hypotheses.
\newblock In {\em Proceedings of the 51st Annual ACM SIGACT Symposium on Theory of Computing}, pages 310--321.

\bibitem[Capp{\'e} et~al., 2013]{KLUCBJournal}
Capp{\'e}, O., Garivier, A., Maillard, O.-A., Munos, R., and Stoltz, G. (2013).
\newblock {{K}ullback-{L}eibler upper confidence bounds for optimal sequential allocation}.
\newblock {\em Annals of Statistics}, 41(3):1516--1541.

\bibitem[Couch et~al., 2019]{couch2019differentially}
Couch, S., Kazan, Z., Shi, K., Bray, A., and Groce, A. (2019).
\newblock Differentially private nonparametric hypothesis testing.
\newblock In {\em Proceedings of the 2019 ACM SIGSAC Conference on Computer and Communications Security}, pages 737--751.

\bibitem[Dwork et~al., 2006]{dwork2006calibrating}
Dwork, C., McSherry, F., Nissim, K., and Smith, A. (2006).
\newblock Calibrating noise to sensitivity in private data analysis.
\newblock In {\em Theory of Cryptography: Third Theory of Cryptography Conference, TCC 2006, New York, NY, USA, March 4-7, 2006. Proceedings 3}, pages 265--284. Springer.

\bibitem[Dwork et~al., 2010]{dwork2010differential}
Dwork, C., Naor, M., Pitassi, T., and Rothblum, G.~N. (2010).
\newblock Differential privacy under continual observation.
\newblock In {\em Proceedings of the forty-second ACM symposium on Theory of computing}, pages 715--724.

\bibitem[Dwork et~al., 2014]{dwork2014algorithmic}
Dwork, C., Roth, A., et~al. (2014).
\newblock The algorithmic foundations of differential privacy.
\newblock {\em Foundations and Trends{\textregistered} in Theoretical Computer Science}, 9(3--4):211--407.

\bibitem[Gaboardi et~al., 2016]{gaboardi2016differentially}
Gaboardi, M., Lim, H., Rogers, R., and Vadhan, S. (2016).
\newblock Differentially private chi-squared hypothesis testing: Goodness of fit and independence testing.
\newblock In {\em International conference on machine learning}, pages 2111--2120. PMLR.

\bibitem[Garivier and Kaufmann, 2016]{GK16}
Garivier, A. and Kaufmann, E. (2016).
\newblock Optimal best arm identification with fixed confidence.
\newblock In {\em Proceedings of the 29th Conference On Learning Theory}.

\bibitem[Garivier et~al., 2019]{GMS16}
Garivier, A., M{\'{e}}nard, P., and Stoltz, G. (2019).
\newblock Explore first, exploit next: The true shape of regret in bandit problems.
\newblock {\em Mathemathics of Opereration Research}, 44(2):377--399.

\bibitem[Kalogerias et~al., 2021]{kalogerias2020best}
Kalogerias, D.~S., Nikolakakis, K.~E., Sarwate, A.~D., and Sheffet, O. (2021).
\newblock Quantile multi-armed bandits: Optimal best-arm identification and a differentially private scheme.
\newblock {\em IEEE Journal on Selected Areas in Information Theory}, 2(2):534--548.

\bibitem[Kamath et~al., 2022]{kamath2022private}
Kamath, G., Mouzakis, A., Singhal, V., Steinke, T., and Ullman, J. (2022).
\newblock A private and computationally-efficient estimator for unbounded gaussians.
\newblock In {\em Conference on Learning Theory}, pages 544--572. PMLR.

\bibitem[Kaplan et~al., 2021]{kaplan2021sparse}
Kaplan, H., Mansour, Y., and Stemmer, U. (2021).
\newblock The sparse vector technique, revisited.
\newblock In {\em Conference on Learning Theory}, pages 2747--2776. PMLR.

\bibitem[Kazan et~al., 2023]{kazan2023test}
Kazan, Z., Shi, K., Groce, A., and Bray, A.~P. (2023).
\newblock The test of tests: A framework for differentially private hypothesis testing.
\newblock In {\em International Conference on Machine Learning}, pages 16131--16151. PMLR.

\bibitem[Mironov, 2017]{mironov2017renyi}
Mironov, I. (2017).
\newblock R{\'e}nyi differential privacy.
\newblock In {\em 2017 IEEE 30th computer security foundations symposium (CSF)}, pages 263--275. IEEE.

\bibitem[Papernot and Steinke, 2021]{papernot2021hyperparameter}
Papernot, N. and Steinke, T. (2021).
\newblock Hyperparameter tuning with renyi differential privacy.
\newblock {\em arXiv preprint arXiv:2110.03620}.

\bibitem[Sajed and Sheffet, 2019]{dpseOrSheffet}
Sajed, T. and Sheffet, O. (2019).
\newblock An optimal private stochastic-mab algorithm based on optimal private stopping rule.
\newblock In {\em International Conference on Machine Learning}, pages 5579--5588. PMLR.

\bibitem[Siegmund, 2013]{siegmund2013sequential}
Siegmund, D. (2013).
\newblock {\em Sequential analysis: tests and confidence intervals}.
\newblock Springer Science \& Business Media.

\bibitem[Steinke, 2022]{steinke2022composition}
Steinke, T. (2022).
\newblock Composition of differential privacy \& privacy amplification by subsampling.
\newblock {\em arXiv preprint arXiv:2210.00597}.

\bibitem[Tirinzoni et~al., 2022]{Tirinzoni22EPRL}
Tirinzoni, A., Marjani, A.~A., and Kaufmann, E. (2022).
\newblock Near instance-optimal {PAC} reinforcement learning for deterministic mdps.
\newblock In {\em Advances in Neural Information Processing Systems (NeurIPS)}.

\bibitem[Wald, 1945]{Wald45SPRT}
Wald, A. (1945).
\newblock {Sequential Tests of Statistical Hypotheses}.
\newblock {\em Annals of Mathematical Statistics}, 16(2):117--186.

\bibitem[Wald, 1948]{Wald48OptSPRT}
Wald, A. (1948).
\newblock {Optimum character of the Sequential Probability Ratio Test}.
\newblock {\em Annals of Mathematical Statistics}, 19(3):326--339.

\bibitem[Wald, 2004]{wald2004sequential}
Wald, A. (2004).
\newblock {\em Sequential analysis}.
\newblock Courier Corporation.

\bibitem[Wang et~al., 2022]{wang2022differentially}
Wang, Y., Sibai, H., Yen, M., Mitra, S., and Dullerud, G.~E. (2022).
\newblock Differentially private algorithms for statistical verification of cyber-physical systems.
\newblock {\em IEEE Open Journal of Control Systems}, 1:294--305.

\bibitem[Zhang et~al., 2022]{Cummings22DPSPRT}
Zhang, W., Mei, Y., and Cummings, R. (2022).
\newblock Private sequential hypothesis testing for statisticians: Privacy, error rates, and sample size.
\newblock In {\em {AISTATS}}.

\bibitem[Zhu and Wang, 2020]{zhu2020improving}
Zhu, Y. and Wang, Y.-X. (2020).
\newblock Improving sparse vector technique with renyi differential privacy.
\newblock {\em Advances in neural information processing systems}, 33:20249--20258.

\end{thebibliography}
\newpage
\appendix
\part{Appendix}
\parttoc
\newpage
\section{A Primer on SPRT for Exponential Families}
\label{appendix:sprt-exponential-families}

A one-dimensional canonical exponential family is a class of distributions
\[\cF = \{ (\nu_{\theta})_{\theta \in \Theta} : \nu_{\theta} \text{ has a density } \ f_{\theta}(x) = \exp(\theta x - b(\theta)) \text{ w.r.t. } \nu \},\]
where $\nu$ is the reference measure, $\Theta$ is the natural parameter space and $b : \Theta \rightarrow \R$ is a convex, twice differentiable function called the log-partition function.

For exponential families, it can be shown that $\mu(\theta) = \bE_{X \sim \nu_{\theta}}[X]$ satisfies $\mu(\theta) = b'(\theta)$, and $\sigma^2(\theta) = \mathrm{Var}_{X\sim \nu_{\theta}}[X]$ satisfies $\sigma^2(\theta) = b''(\theta)$. As a consequence, the mean is an increasing function of the natural parameter, and there is a one-to-one mapping between mean and natural parameter.

We recall the expression of the KL divergence in an exponential family (see, e.g., \cite{KLUCBJournal}). For all $\theta,\theta' \in \Theta$, we have
\begin{equation}\KL(\nu_{\theta},\nu_{\theta'}) = \bE_{X \sim \nu_{\theta}} \left[\log \frac{f_{\theta}(X)}{f_{\theta'}(X)}\right] = \mu(\theta) ( \theta - \theta') - b(\theta) + b(\theta')\;.\label{def:KL_exponentials}\end{equation}

\paragraph{Specific form of the SPRT.} The general SPRT
\[\tau = \inf \left\{ n \in \N : \prod_{i=1}^{n}\frac{f_{\theta_1}(X_i)}{f_{\theta_0}(X_i)} \notin \left(\beta,1/\alpha\right)\right\}\]
particularizes to
\[\tau = \inf \left\{ n \in \N : \exp\left((\theta_1-\theta_0)\sum_{i=1}^{n}X_i - n (b(\theta_1) - b(\theta_0))\right) \notin \left(\beta,1/\alpha\right)\right\},\]
which can be re-written $\tau = \min (\tau_0,\tau_1)$ with
\begin{eqnarray*}
    \tau_0 &=& \inf \left\{ n \in \N : (\theta_1-\theta_0)\sum_{i=1}^{n}X_i - n (b(\theta_1) - b(\theta_0)) < \log(\beta)\right\},\\
    \tau_1 &=& \inf \left\{ n \in \N : (\theta_1-\theta_0)\sum_{i=1}^{n}X_i - n (b(\theta_1) - b(\theta_0)) > \log(1/\alpha)\right\}.
\end{eqnarray*}
Dividing by $\theta_1 -\theta_0 > 0$ and re-arranging the terms, we get

\begin{eqnarray*}
    \tau_0 &=& \inf\left\{ n \in \N :\frac{\sum_{i=1}^{n}X_i}{n} < \frac{b(\theta_1) - b(\theta_0)}{\theta_1 - \theta_0} - \frac{\frac{1}{n} \log(1/\beta)}{\theta_1 - \theta_0}\right\},\\
    \tau_1 &=& \inf\left\{ n \in \N :\frac{\sum_{i=1}^{n}X_i}{n} > \frac{b(\theta_1) - b(\theta_0)}{\theta_1 - \theta_0} + \frac{\frac{1}{n} \log(1/\alpha)}{\theta_1 - \theta_0}\right\}.
\end{eqnarray*}
The expression given in Section~\ref{sec:prelim} can be obtained by noting that $b(\theta_1) - b(\theta_0) = \mu_0(\theta_1 - \theta_0) + \KL(\nu_{\theta_0}, \nu_{\theta_1})$ and $b(\theta_1) - b(\theta_0) = \mu_1(\theta_1 - \theta_0) - \KL(\nu_{\theta_1},\nu_{\theta_0})$.

However, the above expression is also interesting, as it shows that the two thresholds
\[T_{0}^n = \frac{b(\theta_1) - b(\theta_0)}{\theta_1 - \theta_0} - \frac{\frac{1}{n} \log(1/\beta)}{\theta_1 - \theta_0} \text{ and } T_1^n =\frac{b(\theta_1) - b(\theta_0)}{\theta_1 - \theta_0} + \frac{\frac{1}{n} \log(1/\alpha)}{\theta_1 - \theta_0}\]
to which the empirical mean is compared indeed satisfies $T_0^n \leq T_{1}^n$.

DP-SPRT can thus be alternatively written
\begin{eqnarray*}
    \tau_0 &=& \inf \left\{ n \in \N :\frac{\sum_{i=1}^{n}X_i}{n} + \frac{Y_n}{n}< \frac{b(\theta_1) - b(\theta_0)}{\theta_1 - \theta_0} - \frac{\frac{1}{n} \log(1/\gamma\beta)}{\theta_1 - \theta_0} - C(n,(1-\gamma)\beta) - \frac{Z}{n}\right\},\\
    \tau_1 &=& \inf \left\{ n \in \N :\frac{\sum_{i=1}^{n}X_i}{n} + \frac{Y_n}{n} > \frac{b(\theta_1) - b(\theta_0)}{\theta_1 - \theta_0} + \frac{\frac{1}{n} \log(1/\gamma\alpha)}{\theta_1 - \theta_0} + C(n,(1-\gamma)\alpha)+\frac{Z}{n}\right\}.
\end{eqnarray*}

\paragraph{Bernoulli distributions.} Bernoulli distribution form a one-dimensional exponential family if we set the natural parameter of the Bernoulli distribution with probability of success $p_i$ to be $\theta_i = \log \frac{p_i}{1 - p_i}$. The log-partition function is $b(\theta) = \log(1+e^{\theta})$.

\paragraph{On \dpsprt{} versus PrivSPRT.} PrivSPRT \citep{Cummings22DPSPRT} does not rely on the specific form of SPRT for exponential family based on its sufficient statistic (the empirical mean). It directly privatizes the sum of (truncated) log-likelihood ratios. For Bernoulli distribution with means $p_0$ and $p_1$, the stopping rule of PrivSPRT is $\widetilde{\tau} = \min(\widetilde{\tau}_0,\widetilde{\tau}_1)$ with

\begin{eqnarray*}
    \widetilde{\tau}_0 &=& \inf \left\{ n \in \N : \sum_{i=1}^{n}\left[X_i \log \frac{p_1}{p_0} + (1-X_i) \log \frac{1 - p_1}{1-p_0}\right]_{-A}^{A} + Y_n^{1} > b + Z^{1} \right\},\\
    \widetilde{\tau}_1 &=& \inf \left\{ n \in \N : \sum_{i=1}^{n}\left[X_i \log \frac{p_1}{p_0} + (1-X_i) \log \frac{1 - p_1}{1-p_0}\right]_{-A}^{A} + Y_n^{2} < -a + Z^{2} \right\},
\end{eqnarray*}
where $Y_n^{1},Y_n^{2}$ are iid from $\cN(0,\sigma_2^2)$, and $Z^{1},Z^{2}$ are both drawn from $\cN(0,\sigma_1^2)$. $[x]_{-A}^{A}$ stands for the truncation to $[-A,A]$. On some instances on which the truncation is immaterial, that is when
\[ \max\left(\log \frac{p_1}{p_0}  ,  \log \frac{1-p_0}{1-p_1}\right) \leq A\]
(which is the case for $p_0=0.3, p_1=0.7$ and $A=1$), PrivSPRT can also be re-expressed with the empirical mean. We obtain
\begin{eqnarray*}
    \widetilde{\tau}_0 &=& \inf \left\{ n \in \N : \frac{\sum_{i=1}^{n}X_i}{n} + \frac{Y_n^{1}}{n(\theta_1 - \theta_0)} > \frac{b(\theta_1) - b(\theta_0)}{\theta_1 - \theta_0} + \frac{b}{n(\theta_1 - \theta_0)} + \frac{Z^1}{n(\theta_1 - \theta_0)}\right\},\\
    \widetilde{\tau}_1 &=& \inf \left\{ n \in \N : \frac{\sum_{i=1}^{n}X_i}{n} + \frac{Y_n^{2}}{n(\theta_1 - \theta_0)} < \frac{b(\theta_1) - b(\theta_0)}{\theta_1 - \theta_0} - \frac{a}{n(\theta_1 - \theta_0)} + \frac{Z^2}{n(\theta_1 - \theta_0)}\right\}.
\end{eqnarray*}
The differences with \dpsprt{} are the following: (1) PrivSPRT uses $a$ and $b$ as parameters to be tuned (empirically) to attain desired error levels while \dpsprt{} uses an exact, theoretically valid calibration. (2) PrivSPRT uses two streams of noise, whereas the noise are common for $\tau_0$ and $\tau_1$ in \dpsprt{}. (3) The noise is not scaled in the same way in both approaches.

\section{Privacy of \outint: Proof of Theorem~\ref{thm:general-outside-interval}}
\label{appendix:privacy-proofs}

\textbf{Two useful lemmas.} We will use the following two lemmas in the proof of Theorem \ref{thm:general-outside-interval}. The first lemma is a direct consequence of the definition of R\'enyi differential privacy.
\begin{lemma}[Indistinguishability Property of RDP (from \cite{mironov2017renyi})]
    \label{lem:indistinguishability}
    Let $\mathcal{M}$ be a mechanism that satisfies $(\alpha, \tilde{\epsilon}(\alpha))$-RDP. Then for any measurable set $S \subseteq \text{Range}(\mathcal{M})$ and any neighboring datasets $D, D'$:
    \begin{equation}
        e^{-\tilde{\epsilon}(\alpha) \frac{\alpha}{\alpha-1}} \cdot \Pr[\mathcal{M}(D') \in S]^{\frac{\alpha}{\alpha-1}} \leq \Pr[\mathcal{M}(D) \in S] \leq e^{\tilde{\epsilon}(\alpha) \frac{\alpha}{\alpha-1}} \cdot \Pr[\mathcal{M}(D') \in S]^{\frac{\alpha-1}{\alpha}}.
    \end{equation}
\end{lemma}

This second lemma is a direct consequence of the definition of differential privacy for noise-adding mechanisms with continuous noise.

\begin{lemma}
    \label{lem:density-ratio}
    Let $M(D) = f(D) + Z$ be a mechanism that adds noise $Z \in \mathbb{R}$ with probability density function $p_Z$ to the output of a query $f$. Let $\Delta$ be the sensitivity of $f$, i.e., $\Delta = \max_{D \sim D'} |f(D) - f(D')|$, where $D \sim D'$ denotes that $D$ and $D'$ are adjacent datasets. If $M$ is $\varepsilon$-differentially private, then
    \begin{equation*}
        \forall z \in \R, \   {p_Z(z - \Delta)} \leq e^{\varepsilon}{p_Z(z)}\;.
    \end{equation*}
\end{lemma}

\begin{proof}
    By definition, $M$ is $\varepsilon$-differentially private if for all adjacent datasets $D \sim D'$ and all measurable sets $S \subseteq \mathbb{R}$:
    \begin{equation*}
        \Pr[M(D) \in S] \leq e^{\varepsilon} \Pr[M(D') \in S].
    \end{equation*}

    This must hold for all measurable sets $S$, which implies that it must hold pointwise for the probability density functions:
    \begin{equation*}
        \forall x \in \R, \  p_{M(D)}(x) \leq e^{\varepsilon} p_{M(D')}(x)\;.
    \end{equation*}

    Since $p_{M(D)}(x) = p_Z(x - f(D))$ and $p_{M(D')}(x) = p_Z(x - f(D'))$, we have, for all query $f$ of sensitivity $\Delta$ and for all neighboring $D$ and $D'$,
    \begin{equation*}
        \forall x\in \R, \   p_Z(x - f(D)) \leq e^{\varepsilon} p_Z(x - f(D'))\;.
    \end{equation*}

    Consider the worst case where $f(D) = f(D') + \Delta$. Substituting we get
    \begin{equation*}
        \forall x \in \R, \   p_Z(x - f(D') - \Delta) \leq e^{\varepsilon} p_Z(x - f(D'))\;.
    \end{equation*}

    Letting $z = x - f(D')$ yields
    \begin{equation*}
        \forall z \in \R, \  p_Z(z - \Delta) \leq e^{\varepsilon} p_Z(z)\;.
    \end{equation*}
\end{proof}

We prove below a more general version of Theorem \ref{thm:general-outside-interval} for the \outint{} algorithm. Theorem \ref{thm:general-outside-interval} can be obtained by setting $\gamma=2$ in the following theorem.

\begin{theorem}[General Form of Theorem \ref{thm:general-outside-interval}]
    Let $\tau$ be a random variable indicating the stopping time of the \outint{} mechanism.

    (i) If the noise-adding mechanism associated with $\cD_{Z}$ satisfies $\epsilon_Z$-DP, and the mechanism associated with $\cD_Y$ satisfies $\epsilon_Y$-DP for queries with sensitivity $2\Delta$, then \outint{} is $(\epsilon_Z + \epsilon_Y)$-DP.

    (ii) If the noise-adding mechanism associated with $\cD_{Z}$ has an RDP profile $\epsilon_{Z}(\alpha)$ for queries with sensitivity $\Delta$ and the one associated with $\cD_Y$ has an RDP profile $\epsilon_Y(\alpha)$ for queries with sensitivity $2\Delta$, then \outint{} (denoted by $\mathcal{A}$) satisfies:
    \begin{equation*}
        \mathbb{D}_{\alpha}(\mathcal{A}(D)\|\mathcal{A}(D')) \leq \frac{\alpha - (\gamma - 1)/\gamma}{\alpha - 1}\epsilon_{Z}\left(\frac{\gamma}{\gamma - 1}\alpha\right) + \epsilon_{Y}(\alpha) + \frac{\log \left(\mathbb{E}_{z \sim \cD_{Z}} [\mathbb{E}[\tau|z]^{\gamma}]\right)}{\gamma(\alpha - 1)},
    \end{equation*}
    for any $\gamma>1$.
\end{theorem}

The proof is adapted from the one by \cite{zhu2020improving} for GeneralAboveThreshold.

Consider an arbitrary output of the algorithm $\mathbf{a}= (\bot,\dots,\bot,\top_d)$ with $\bot$ appearing $k-1$ times and $d \in \{0,1\}$ indicating whether the $k$-th query falls below the lower threshold ($d=0$) or exceeds the upper threshold ($d=1$). Let
\begin{align*}
    g_i(D,z)       & = \Pr[f_i(D) + Y_i > T_0^i - z \land f_i(D) + Y_i < T_1^i + z], \\
    h_k^{(d)}(D,z) & = \begin{cases}
                           \Pr[f_k(D) + Y_k \geq T_1^k + z] & \text{if } d=1 \\
                           \Pr[f_k(D) + Y_k \leq T_0^k - z] & \text{if } d=0
                       \end{cases},
\end{align*}
where the randomness in $\Pr$ is on the distribution of the $Y$s only. Letting $p_{Z}$ be the pdf of $Z$, we have
\begin{equation*}
    \Pr[\mathcal{A}(D) = \mathbf{a}] = \int_{-\infty}^{\infty} p_{Z}(z) \prod_{i=1}^{k-1} g_i(D,z) \cdot h_k^{(d)}(D,z) \, dz\;.
\end{equation*}

\noindent\textbf{Step 1: Bounding output probabilities.}

Using a change of variable $z \to z-\Delta$:
\begin{align*}
    \Pr[\mathcal{A}(D) = \mathbf{a}] & = \int_{-\infty}^{\infty} p_{Z}(z-\Delta) \prod_{i=1}^{k-1} g_i(D,z-\Delta) \cdot h_k^{(d)}(D,z-\Delta) \, dz                              \\
                                     & \leq \int_{-\infty}^{\infty} \frac{p_{Z}(z-\Delta)}{p_{Z}(z)}p_{Z}(z) \prod_{i=1}^{k-1} g_i(D,z-\Delta) \cdot h_k^{(d)}(D,z-\Delta) \, dz.
\end{align*}
Note that $p_{Z}(z-\Delta)/p_{Z}(z)$ is well-defined for DP or RDP mechanism, as from Lemma~\ref{lem:indistinguishability} or  Lemma~\ref{lem:density-ratio} it can be shown that the density has to be supported on $\R$.

For queries before $k$, we establish an upper bound on $g_i(D,z-\Delta)$ through inequality analysis. For any $i < k$, we have:
\begin{align*}
    g_i(D,z-\Delta) & = \Pr[f_i(D) + Y_i > T_0^i - (z-\Delta) \land f_i(D) + Y_i < T_1^i + (z-\Delta)].
\end{align*}

Since $f_i$ has sensitivity $\Delta$, we know that $|f_i(D) - f_i(D')| \leq \Delta$. This gives us two inequalities:
\begin{align*}
    f_i(D) & \leq f_i(D') + \Delta, \\
    f_i(D) & \geq f_i(D') - \Delta.
\end{align*}

For the first inequality condition, we have:
\begin{align*}
    \Pr[f_i(D) + Y_i > T_0^i - (z-\Delta)] & \leq \Pr[f_i(D') + \Delta + Y_i > T_0^i - (z-\Delta)] \\
                                           & = \Pr[f_i(D') + Y_i > T_0^i - z].
\end{align*}

For the second inequality condition:
\begin{align*}
    \Pr[f_i(D) + Y_i < T_1^i + (z-\Delta)] & \leq \Pr[f_i(D') - \Delta + Y_i < T_1^i + (z-\Delta)] \\
                                           & = \Pr[f_i(D') + Y_i < T_1^i + z].
\end{align*}

Since both conditions in the conjunction are upper-bounded, we obtain:
\begin{align*}
    g_i(D,z-\Delta) & \leq \Pr[f_i(D') + Y_i > T_0^i - z \land f_i(D') + Y_i < T_1^i + z] \\
                    & = g_i(D',z).
\end{align*}

Therefore:
\begin{align}
    \Pr[\mathcal{A}(D) = \mathbf{a}] & \leq \int_{-\infty}^{\infty} \frac{p_{Z}(z-\Delta)}{p_{Z}(z)}p_{Z}(z) \prod_{i=1}^{k-1} g_i(D',z) \cdot h_k^{(d)}(D,z-\Delta) \, dz \nonumber       \\
                                     & = \bE_{z\sim \cD_{Z}}\left[\frac{p_{Z}(z-\Delta)}{p_{Z}(z)} \prod_{i=1}^{k-1} g_i(D',z) \cdot h_k^{(d)}(D,z-\Delta)\right]\;. \label{eq:probabound}
\end{align}

With our notation, we can also write
\begin{equation}
    \Pr[\mathcal{A}(D') = \mathbf{a}] = \bE_{z\sim \cD_{Z}}\left[\prod_{i=1}^{k-1} g_i(D',z) \cdot h^{(d)}_k(D',z)\right]\;. \label{eq:probabound2}
\end{equation}

We will now use the properties of the private mechanism to relate the two probabilities. First we will prove the results when the mechanism is $\epsilon$-DP.

\subsection{$\epsilon$-DP Guarantee}
\begin{proof}[Proof of Theorem \ref{thm:general-outside-interval} (i)]
    We have established in Equation \eqref{eq:probabound} that for any output $\mathbf{a} = (\bot,\ldots,\bot,\top_d)$ with $\bot$ appearing $k-1$ times:
    \begin{align*}
        \Pr[\mathcal{A}(D) = \mathbf{a}] & \leq \bE_{z\sim \cD_{Z}}\left[\frac{p_{Z}(z-\Delta)}{p_{Z}(z)} \prod_{i=1}^{k-1} g_i(D',z) \cdot h_k^{(d)}(D,z-\Delta)\right].
    \end{align*}

    \textbf{Step 2: Bounding the density ratio and stopping time query.}

    Since the noise-adding mechanism corresponding to $\cD_Z$ satisfies $\epsilon_Z$-DP for queries with sensitivity $\Delta$, by Lemma \ref{lem:density-ratio}, we have:
    \begin{align*}
        \frac{p_{Z}(z-\Delta)}{p_{Z}(z)} \leq e^{\epsilon_Z}\;.
    \end{align*}

    For $d \in \{0,1\}$, we need to establish a bound on $h_k^{(d)}(D,z-\Delta)$ in terms of $h_k^{(d)}(D',z)$. We first consider the case $d=1$ (upper threshold crossed).

    Let us define a new query:
    \begin{equation*}
        \tilde{f}_k(D'') =
        \begin{cases}
            f_k(D) + \Delta & \text{if } D''=D  \\
            f_k(D')         & \text{if } D''=D'
        \end{cases}.
    \end{equation*}

    To establish the sensitivity of $\tilde{f}_k$, we note that since $f_k$ has sensitivity $\Delta$, we have $|f_k(D) - f_k(D')| \leq \Delta$. Thus:
    \begin{align*}
        |\tilde{f}_k(D) - \tilde{f}_k(D')| & = |f_k(D) + \Delta - f_k(D')|    \\
                                           & \leq |f_k(D) - f_k(D')| + \Delta \\
                                           & \leq \Delta + \Delta = 2\Delta.
    \end{align*}

    Therefore, $\tilde{f}_k$ has sensitivity $2\Delta$. Now, we can rewrite:
    \begin{align*}
        h_k^{(1)}(D,z-\Delta) & = \Pr[f_k(D) + Y_k \geq T_1^k + (z-\Delta)]                  \\
                              & = \Pr[\tilde{f}_k(D) - \Delta + Y_k \geq T_1^k + (z-\Delta)] \\
                              & = \Pr[\tilde{f}_k(D) + Y_k \geq T_1^k + z].
    \end{align*}

    Similarly, $h_k^{(1)}(D',z) = \Pr[\tilde{f}_k(D') + Y_k \geq T_1^k + z]$.

    Since the noise-adding mechanism corresponding to $\mathcal{D}_Y$ satisfies $\epsilon_Y$-DP for queries with sensitivity $2\Delta$, we can apply the privacy guarantee:
    \begin{align*}
        \Pr[\tilde{f}_k(D) + Y_k \geq T_1^k + z] & \leq e^{\epsilon_Y} \Pr[\tilde{f}_k(D') + Y_k \geq T_1^k + z].
    \end{align*}

    Therefore: $h_k^{(1)}(D,z-\Delta) \leq e^{\epsilon_Y} h_k^{(1)}(D',z)$.

    A similar analysis for $d=0$ (lower threshold crossed) would define $\tilde{f}_k(D) = f_k(D) - \Delta$ and $\tilde{f}_k(D') = f_k(D')$, yielding the same bound: $h_k^{(0)}(D,z-\Delta) \leq e^{\epsilon_Y} h_k^{(0)}(D',z)$.

    \textbf{Step 3: Deriving the final DP bound.}

    Using these inequalities in our original expression:
    \begin{align*}
        \Pr[\mathcal{A}(D) = \mathbf{a}] & \leq \bE_{z\sim \cD_{Z}}\left[\frac{p_{Z}(z-\Delta)}{p_{Z}(z)} \prod_{i=1}^{k-1} g_i(D',z) \cdot h_k^{(d)}(D,z-\Delta)\right] \\
                                         & \leq \bE_{z\sim \cD_{Z}}\left[e^{\epsilon_Z} \prod_{i=1}^{k-1} g_i(D',z) \cdot h_k^{(d)}(D,z-\Delta)\right]                   \\
                                         & \leq \bE_{z\sim \cD_{Z}}\left[e^{\epsilon_Z} \prod_{i=1}^{k-1} g_i(D',z) \cdot e^{\epsilon_Y}h_k^{(d)}(D',z)\right]           \\
                                         & = e^{\epsilon_Z + \epsilon_Y} \bE_{z\sim \cD_{Z}}\left[\prod_{i=1}^{k-1} g_i(D',z) \cdot h_k^{(d)}(D',z)\right].
    \end{align*}

    From Equation \eqref{eq:probabound2}, we recognize that the expectation in the last line equals $\Pr[\mathcal{A}(D') = \mathbf{a}]$, giving us:
    \begin{align*}
        \Pr[\mathcal{A}(D) = \mathbf{a}] & \leq e^{\epsilon_Z + \epsilon_Y} \Pr[\mathcal{A}(D') = \mathbf{a}].
    \end{align*}

    By symmetry of the definitions of neighboring datasets, we can also show that $\Pr[\mathcal{A}(D') = \mathbf{a}] \leq e^{\epsilon_Z + \epsilon_Y} \Pr[\mathcal{A}(D) = \mathbf{a}]$.

    Since this bound holds for any output $\mathbf{a}$, we have established that \outint{} satisfies $(\epsilon_Z + \epsilon_Y)$-differential privacy.
\end{proof}

\subsection{\texorpdfstring{$(\alpha, \epsilon)$}{(alpha, epsilon)}-R\'enyi DP Guarantee}
\begin{proof}[Proof of Theorem \ref{thm:general-outside-interval} (ii)]
    Starting again from Equation \eqref{eq:probabound}, we proceed with proving the RDP guarantee.

    \textbf{Step 2: Computing the Rényi divergence and applying Jensen's inequality.}

    Let us consider the expectation in the definition of Rényi divergence:
    \begin{align*}
        \mathbb{E}_{\mathbf{a} \sim \mathcal{A}(D')} \left[ \frac{\Pr[\mathcal{A}(D) = \mathbf{a}]^{\alpha}}{\Pr[\mathcal{A}(D') = \mathbf{a}]^{\alpha}} \right] & = \sum_{\mathbf{a}\in \mathcal O} \frac{\Pr[\mathcal{A}(D) = \mathbf{a}]^{\alpha}}{\Pr[\mathcal{A}(D') = \mathbf{a}]^{\alpha-1}}                                                                                                                                               \\
                                                                                                                                                                 & \leq \sum_{k=1}^\infty \sum_{d\in\{0,1\}}\frac{\bE_{z\sim \cD_{Z}}\left[\frac{p_{Z}(z-\Delta)}{p_{Z}(z)} \prod_{i=1}^{k-1} g_i(D',z) \cdot h_k^{(d)}(D,z-\Delta)\right]^\alpha}{\bE_{z\sim \cD_{Z}}\left[\prod_{i=1}^{k-1} g_i(D',z) \cdot h^{(d)}_k(D',z)\right]^{\alpha-1}},
    \end{align*}
    where we obtain the inequality using Equations \eqref{eq:probabound} and \eqref{eq:probabound2}.

    Let us consider the bivariate function $\phi_\alpha:\bR^2_+ \to \bR_+$ defined as $\phi_\alpha(x,y) = \frac{x^\alpha}{y^{\alpha-1}}$. Notice that for any $\alpha > 1$, $\phi_\alpha$ is jointly convex over $\bR^2_+$. By Jensen's inequality, we have $\phi_\alpha(\bE[X],\bE[Y]) \leq \bE[\phi_\alpha(X,Y)]$ for any positive random variables $X$ and $Y$. Applying this inequality to our case:

    \begin{align*}
         & \mathbb{E}_{\mathbf{a} \sim \mathcal{A}(D')} \left[ \frac{\Pr[\mathcal{A}(D) = \mathbf{a}]^{\alpha}}{\Pr[\mathcal{A}(D') = \mathbf{a}]^{\alpha}} \right]                                                                                                                                                                                          \\
         & \leq \bE_{z\sim \cD_{Z}}\left[\sum_{k=1}^\infty \sum_{d\in\{0,1\}}\frac{\left(\frac{p_{Z}(z-\Delta)}{p_{Z}(z)} \prod_{i=1}^{k-1} g_i(D',z) \cdot h_k^{(d)}(D,z-\Delta)\right)^\alpha}{\left(\prod_{i=1}^{k-1} g_i(D',z) \cdot h^{(d)}_k(D',z)\right)^{\alpha-1}}\right]                                                                           \\
         & = \bE_{z\sim \cD_{Z}}\left[\left(\frac{p_{Z}(z-\Delta)}{p_{Z}(z)}\right)^\alpha\bE_{(\tau,\hat{d})\sim\mathcal{A}(D')}\left[\left.\frac{\left( \prod_{i=1}^{\tau-1} g_i(D',z) \cdot h_\tau^{(\hat{d})}(D,z-\Delta)\right)^\alpha}{\left(\prod_{i=1}^{\tau-1} g_i(D',z) \cdot h^{(\hat{d})}_\tau(D',z)\right)^{\alpha}}\right| Z = z\right]\right] \\
         & = \bE_{z\sim \cD_{Z}}\left[\left(\frac{p_{Z}(z-\Delta)}{p_{Z}(z)}\right)^\alpha\bE_{(\tau,\hat{d})\sim\mathcal{A}(D')}\left[\left.\frac{\left(h_\tau^{(\hat{d})}(D,z-\Delta)\right)^\alpha}{\left(h^{(\hat{d})}_\tau(D',z)\right)^{\alpha}}\right| Z = z\right]\right].
    \end{align*}

    \textbf{Step 3: Applying RDP guarantees and relating to stopping time.}

    We now establish a relationship between $h_k^{(d)}(D,z-\Delta)$ and $h^{(d)}_k(D',z)$ under Rényi DP for any $d\in \{0,1\}$ and $k>0$. We prove that
    \begin{align}
        \frac{\left(h^{(d)}_k(D,z-\Delta)\right)^\alpha}{\left(h^{(d)}_k(D',z)\right)^{\alpha-1}} \leq e^{\epsilon_Y(\alpha)(\alpha-1)}\;.
        \label{eq:useful-renyi}
    \end{align}
    To prove it for $d=1$ (the upper threshold case) we introduce the query
    \begin{equation*}
        \tilde{f}_k(D'')=
        \begin{cases}
            f_k(D) + \Delta & \text{if } D''=D     \\
            f_k(D')         & \text{if } D''=D'\;.
        \end{cases}
    \end{equation*}

    To establish the sensitivity of $\tilde{f}_k$, we note that since $f_k$ has sensitivity $\Delta$, we have $|f_k(D) - f_k(D')| \leq \Delta$. Thus:
    \begin{align*}
        |\tilde{f}_k(D) - \tilde{f}_k(D')| & = |f_k(D) + \Delta - f_k(D')|    \\
                                                 & \leq |f_k(D) - f_k(D')| + \Delta \\
                                                 & \leq \Delta + \Delta = 2\Delta.
    \end{align*}

    Therefore, $\tilde{f}_k$ has sensitivity $2\Delta$. Applying the indistinguishability property of Rényi DP (Lemma \ref{lem:indistinguishability}) to the noise-adding mechanism with noise distribution $Y_k$, we obtain
    \begin{align*}
        \frac{\left(\Pr[\tilde{f}_k(D) + Y_k \geq T_1^k + z]\right)^\alpha}{\left(\Pr[\tilde{f}_k(D') + Y_k \geq T_1^k + z]\right)^{\alpha-1}} \leq e^{\epsilon_Y(\alpha)(\alpha-1)}.
    \end{align*}

    Since $\tilde{f}_k(D) + Y_k = f_k(D) + \Delta + Y_k$ and $\tilde{f}_k(D') + Y_k = f_k(D') + Y_k$:
    \begin{align*}
        \frac{\left(\Pr[f_k(D) + \Delta + Y_k \geq T_1^k + z]\right)^\alpha}{\left(\Pr[f_k(D') + Y_k \geq T_1^k + z]\right)^{\alpha-1}} \leq e^{\epsilon_Y(\alpha)(\alpha-1)}.
    \end{align*}

    Since $f_k(D) + \Delta + Y_k \geq T_1^k + z$ is equivalent to $f_k(D) + Y_k \geq T_1^k + (z-\Delta)$:
    \begin{align*}
        \frac{\left(\Pr[f_k(D) + Y_k \geq T_1^k + (z-\Delta)]\right)^\alpha}{\left(\Pr[f_k(D') + Y_k \geq T_1^k + z]\right)^{\alpha-1}} \leq e^{\epsilon_Y(\alpha)(\alpha-1)},
    \end{align*}
    which yields~\eqref{eq:useful-renyi} for $d=1$. The proof for $d=0$ follows the same lines.

    Using Equation~\eqref{eq:useful-renyi},
    \begin{align*}
         & \bE_{(\tau,\hat{d})\sim \cA(D')}\left[\left.\frac{\left(h_\tau^{(\hat{d})}(D,z-\Delta)\right)^\alpha}{\left(h^{(\hat{d})}_\tau(D',z)\right)^{\alpha}}\right| Z=z\right]  \leq \bE_{(\tau,\hat{d})\sim \cA(D')}\left[\left.\frac{e^{(\alpha-1)\epsilon_Y(\alpha)}}{\left(h^{(\hat{d})}_\tau(D',z)\right)}\right| Z=z\right] \\
         & = e^{(\alpha-1)\epsilon_Y(\alpha)} \sum_{k=1}^\infty\sum_{d\in \{0,1\}} \bP_{(\tau,\hat{d})\sim\cA(D')}(\tau = k, \hat{d}=d | Z = z) \frac{1}{h^{(d)}_k(D',z)}                                                                                                                                                             \\
         & = e^{(\alpha-1)\epsilon_Y(\alpha)} \sum_{k=1}^\infty\sum_{d\in \{0,1\}} \prod_{i=1}^{k-1} g_i(D',z) \cdot h^{(d)}_k(D',z) \cdot \frac{1}{h^{(d)}_k(D',z)}                                                                                                                                                                  \\
         & = 2e^{(\alpha-1)\epsilon_Y(\alpha)} \sum_{k=1}^\infty \prod_{i=1}^{k-1} g_i(D',z).
    \end{align*}



    To relate this sum to an expectation, we note that $\prod_{i=1}^{k-1} g_i(D',z)$ represents the probability that the machanism $\cA(D')$ outputs $\bot$ for each of the first $k-1$ queries, which means the stopping time $\tau$ is at least $k$. Therefore:

    \begin{align*}
        \prod_{i=1}^{k-1} g_i(D',z) = \bP_{\cA(D')}(\tau \geq k|Z = z).
    \end{align*}

    Using that for a positive integer-valued random variable $X$, we can express its expectation using $\mathbb{E}[X] = \sum_{k=1}^{\infty} \Pr(X \geq k)$, we have
    %
    \begin{align*}
        \sum_{k=1}^\infty \prod_{i=1}^{k-1} g_i(D',z) & = \sum_{k=1}^\infty \bP_{\cA(D')}(\tau \geq k|Z=z) = \mathbb{E}_{\cA(D')}[\tau|Z=z].
    \end{align*}

    Thus, we have:

    \begin{align*}
        \bE_{\cA(D')}\left[\left.\frac{\left(h_\tau^{(\hat{d})}(D,z-\Delta)\right)^\alpha}{\left(h^{(\hat{d})}_\tau(D',z)\right)^{\alpha}}\right| Z=z\right] & \leq 2e^{(\alpha-1)\epsilon_Y(\alpha)} \cdot \mathbb{E}_{\cA(D')}[\tau|Z=z].
    \end{align*}

    \newpage

    \textbf{Step 4: Combining results and applying Hölder's inequality.}


    Now we have:
    \begin{align*}
        \mathbb{E}_{\mathbf{a} \sim \mathcal{A}(D')} \left[ \frac{\Pr[\mathcal{A}(D) = \mathbf{a}]^{\alpha}}{\Pr[\mathcal{A}(D') = \mathbf{a}]^{\alpha}} \right] & \leq \bE_{z\sim \cD_{Z}}\left[2\left(\frac{p_{Z}(z-\Delta)}{p_{Z}(z)}\right)^\alpha \cdot \bE_{\cA(D')}[\tau|Z=z] \cdot e^{(\alpha-1)\epsilon_Y(\alpha)}\right].
    \end{align*}

    We apply Hölder's inequality to separate the terms involving $Z$ and $\tau$. For any $\gamma > 1$, let $\gamma^*$ be its Hölder conjugate such that $\frac{1}{\gamma^*} + \frac{1}{\gamma} = 1$. Let $X = \left(\frac{p_{Z}(z-\Delta)}{p_{Z}(z)}\right)^{\alpha}$ and $Y = \bE_{\cA(D')}[\tau|Z=z]$. By Hölder's inequality:
    \begin{align*}
        \bE_{z\sim \cD_{Z}}[X \cdot Y] & \leq \left(\bE_{z\sim \cD_{Z}}[X^{\gamma^*}]\right)^{1/\gamma^*} \cdot \left(\bE_{z\sim \cD_{Z}}[Y^{\gamma}]\right)^{1/\gamma}                                                                        \\
                                       & = \left(\bE_{z\sim \cD_{Z}}\left[\left(\frac{p_{Z}(z-\Delta)}{p_{Z}(z)}\right)^{\alpha\gamma^*}\right]\right)^{1/\gamma^*} \cdot \left(\bE_{z\sim \cD_{Z}}[\bE[\tau|Z=z]^{\gamma}]\right)^{1/\gamma}.
    \end{align*}

    Therefore:

    \begin{align*}
        \mathbb{E}_{\mathbf{a} \sim \mathcal{A}(D')} \left[ \frac{\Pr[\mathcal{A}(D) = \mathbf{a}]^{\alpha}}{\Pr[\mathcal{A}(D') = \mathbf{a}]^{\alpha}} \right] & \leq 2e^{(\alpha-1)\epsilon_Y(\alpha)} \cdot \left(\bE_{z\sim \cD_{Z}}\left[\left(\frac{p_{Z}(z-\Delta)}{p_{Z}(z)}\right)^{\alpha\gamma^*}\right]\right)^{1/\gamma^*} \\
                                                                                                                                                                 & \quad \cdot \left(\bE_{z\sim \cD_{Z}}[\bE_{\cA(D')}[\tau|Z=z]^{\gamma}]\right)^{1/\gamma}.
    \end{align*}

    The first expectation term is related to the RDP guarantee of the noise-adding mechanism with distribution $\cD_{Z}$. Specifically, since this mechanism satisfies $\epsilon_{Z}(\alpha\gamma^*)$-RDP for queries with sensitivity $\Delta$, we have:
    \begin{align*}
        \left(\bE_{z\sim \cD_{Z}}\left[\left(\frac{p_{Z}(z-\Delta)}{p_{Z}(z)}\right)^{\alpha\gamma^*}\right]\right)^{1/\gamma^*} & \leq \left(e^{(\alpha\gamma^*-1)\epsilon_{Z}(\alpha\gamma^*)}\right)^{1/\gamma^*} \\
                                                                                                                                 & = e^{\frac{\alpha\gamma^*-1}{\gamma^*}\epsilon_{Z}(\alpha\gamma^*)}.
    \end{align*}

    Substituting $\gamma^* = \frac{\gamma}{\gamma-1}$ (from the Hölder conjugate relation), and simplifying the exponent through algebraic manipulations:
    \begin{align*}
        e^{\frac{\alpha\gamma^*-1}{\gamma^*}\epsilon_{Z}(\alpha\gamma^*)} & = e^{\frac{\alpha\frac{\gamma}{\gamma-1}-1}{\frac{\gamma}{\gamma-1}}\epsilon_{Z}(\alpha\frac{\gamma}{\gamma-1})} \\
                                                                          & = e^{\frac{(\alpha\gamma-\alpha+1)(\gamma-1)}{\gamma}\epsilon_{Z}(\alpha\frac{\gamma}{\gamma-1})}                \\
                                                                          & = e^{\frac{(\alpha-1)\gamma-\alpha+1}{\gamma}\epsilon_{Z}(\alpha\frac{\gamma}{\gamma-1})}                        \\
                                                                          & = e^{\frac{(\alpha-1) - \frac{\alpha-1}{\gamma}}{\alpha-1}(\alpha-1)\epsilon_{Z}(\alpha\frac{\gamma}{\gamma-1})} \\
                                                                          & = e^{\frac{\alpha - (\gamma-1)/\gamma}{\alpha-1}(\alpha-1)\epsilon_{Z}(\alpha\frac{\gamma}{\gamma-1})}.
    \end{align*}

    Taking logarithms and dividing by $(\alpha-1)$, we obtain the bound:
    \begin{align*}
        {D}_{\alpha}(\mathcal{A}(D)\|\mathcal{A}(D')) \leq \frac{\alpha - (\gamma - 1)/\gamma}{\alpha - 1}\epsilon_{Z}\left(\frac{\gamma}{\gamma - 1}\alpha\right) + \epsilon_{Y}(\alpha) + \frac{\log \left(2\mathbb{E}_{z \sim \cD_{Z}} [\mathbb{E}_{\cA(D')}[\tau|Z=z]^{\gamma}]\right)}{\gamma(\alpha - 1)}.
    \end{align*}

    In particular, if we choose $\gamma=2$ then we have:
    \begin{equation*}
        D_{\alpha}(\mathcal{A}(D)\|\mathcal{A}(D')) \leq \frac{\alpha - 1/2}{\alpha - 1}\epsilon_{Z}(2\alpha) + \epsilon_{Y}(\alpha) + \frac{\log \left(2\mathbb{E}_{z \sim \mathcal{D}_{Z}} [\mathbb{E}_{\cA(D')}[\tau|Z=z]^{2}]\right)}{2(\alpha - 1)}.
    \end{equation*}
    This completes the proof of Theorem \ref{thm:general-outside-interval} (ii).
\end{proof}

\section{The Variants of \dpsprt: Laplace and Gaussian Noises}
\label{appendix:implementations}

\subsection{\dpsprt{} with Laplace Noise}

For this variant, we use the following noise distributions and correction function:
\begin{align*}
    Y_n         & \sim \text{Lap}\left(\frac{4}{\epsilon}\right), \\
    Z           & \sim \text{Lap}\left(\frac{2}{\epsilon}\right), \\
    C(n,\delta) & = \frac{6\log(n^s\zeta(s)/\delta)}{n\epsilon}.
\end{align*}

We now prove that this choice satisfies the condition in Equation \eqref{ass:correction}:
\begin{equation*}
    \sum_{n=1}^{\infty} \mathbb{P}\left(\frac{Y_n}{n} - \frac{Z}{n} > C\left(n, \delta\right)\right) \leq \delta.
\end{equation*}

First, we rewrite the inequality inside the probability:
\begin{align*}
    \frac{Y_n}{n} - \frac{Z}{n} & > \frac{6\log(n^s\zeta(s)/\delta)}{n\epsilon} \\
    \Rightarrow Y_n - Z         & > \frac{6\log(n^s\zeta(s)/\delta)}{\epsilon}.
\end{align*}

Using the union bound, we have:
\begin{align*}
    \mathbb{P}\left(Y_n - Z > \frac{6\log(n^s\zeta(s)/\delta)}{\epsilon}\right) & \leq \mathbb{P}\left(Y_n > \frac{4\log(n^s\zeta(s)/\delta)}{\epsilon}\right) + \mathbb{P}\left(Z < -\frac{2\log(n^s\zeta(s)/\delta)}{\epsilon}\right).
\end{align*}

For Laplace random variables, we have the tail bounds:
\begin{align*}
    \text{If } X \sim \text{Lap}(b), \text{ then } \mathbb{P}(X \geq t) = \mathbb{P}(X \leq -t) = \frac{1}{2}e^{-t/b}.
\end{align*}

Therefore:
\begin{align*}
    \mathbb{P}\left(Y_n > \frac{4\log(n^s\zeta(s)/\delta)}{\epsilon}\right) & = \frac{1}{2}e^{-\frac{4\log(n^s\zeta(s)/\delta)}{\epsilon} \cdot \frac{\epsilon}{4}} \\
                                                                            & = \frac{1}{2}e^{-\log(n^s\zeta(s)/\delta)}                                            \\
                                                                            & = \frac{1}{2} \cdot \frac{\delta}{n^s\zeta(s)}                                        \\
                                                                            & = \frac{\delta}{2n^s\zeta(s)}.
\end{align*}

Similarly:
\begin{align*}
    \mathbb{P}\left(Z < -\frac{2\log(n^s\zeta(s)/\delta)}{\epsilon}\right) & = \frac{1}{2}e^{-\frac{2\log(n^s\zeta(s)/\delta)}{\epsilon} \cdot \frac{\epsilon}{2}} \\
                                                                           & = \frac{1}{2}e^{-\log(n^s\zeta(s)/\delta)}                                            \\
                                                                           & = \frac{\delta}{2n^s\zeta(s)}.
\end{align*}

Combining these results:
\begin{align*}
    \mathbb{P}\left(Y_n - Z > \frac{6\log(n^s\zeta(s)/\delta)}{n\epsilon}\right) & \leq \frac{\delta}{2n^s\zeta(s)} + \frac{\delta}{2n^s\zeta(s)} \\
                                                                                 & = \frac{\delta}{n^s\zeta(s)}.
\end{align*}

Summing over all $n$:
\begin{align*}
    \sum_{n=1}^{\infty} \mathbb{P}\left(\frac{Y_n}{n} - \frac{Z}{n} > C\left(n, \delta\right)\right) & \leq \sum_{n=1}^{\infty} \frac{\delta}{n^s\zeta(s)}                 \\
                                                                                                     & = \delta \sum_{n=1}^{\infty} \frac{1}{n^s\zeta(s)}                  \\
                                                                                                     & = \delta \cdot \frac{1}{\zeta(s)} \sum_{n=1}^{\infty} \frac{1}{n^s} \\
                                                                                                     & = \delta \cdot \frac{\zeta(s)}{\zeta(s)}                            \\
                                                                                                     & = \delta,
\end{align*}
where we used the definition of the Riemann zeta function: $\zeta(s) = \sum_{n=1}^{\infty} \frac{1}{n^s}$.

\subsection{\dpsprt{} with Gaussian Noise}

For this variant, we use the following noise distributions and correction function:
\begin{align*}
    Y_n         & \sim \mathcal{N}(0, \sigma_Y^2),                                       \\
    Z           & \sim \mathcal{N}(0, \sigma_Z^2),                                       \\
    C(n,\delta) & = \frac{\sqrt{2(\sigma_Y^2 + \sigma_Z^2)\log(n^s\zeta(s)/\delta)}}{n}.
\end{align*}

We now prove that this choice satisfies the condition in Equation~\eqref{ass:correction}:
\begin{equation*}
    \sum_{n=1}^{\infty} \mathbb{P}\left(\frac{Y_n}{n} - \frac{Z}{n} > C\left(n, \delta\right)\right) \leq \delta.
\end{equation*}

Since $Y_n$ and $Z$ are independent Gaussian random variables, their difference is also Gaussian. Specifically:
\begin{align*}
    Y_n - Z \sim \mathcal{N}(0, \sigma_Y^2 + \sigma_Z^2).
\end{align*}

The probability we need to bound is:
\begin{align*}
    \mathbb{P}\left(\frac{Y_n}{n} - \frac{Z}{n} > \frac{\sqrt{2(\sigma_Y^2 + \sigma_Z^2)\log(n^s\zeta(s)/\delta)}}{n}\right).
\end{align*}

Multiplying both sides by $n$:
\begin{align*}
    \mathbb{P}\left(Y_n - Z > \sqrt{2(\sigma_Y^2 + \sigma_Z^2)\log(n^s\zeta(s)/\delta)}\right).
\end{align*}

For a Gaussian random variable $X \sim \mathcal{N}(0, \sigma^2)$, we have the standard tail bound:
\begin{align*}
    \mathbb{P}(X > t) \leq e^{-t^2/2\sigma^2}.
\end{align*}

Applying this bound with $X = Y_n - Z$, $\sigma^2 = \sigma_Y^2 + \sigma_Z^2$, and $t = \sqrt{2(\sigma_Y^2 + \sigma_Z^2)\log(n^s\zeta(s)/\delta)}$:
\begin{align*}
    \mathbb{P}(Y_n - Z > t) & \leq e^{-t^2/2(\sigma_Y^2 + \sigma_Z^2)}                                             \\
                            & = e^{-2(\sigma_Y^2 + \sigma_Z^2)\log(n^s\zeta(s)/\delta)/2(\sigma_Y^2 + \sigma_Z^2)} \\
                            & = e^{-\log(n^s\zeta(s)/\delta)}                                                      \\
                            & = \frac{\delta}{n^s\zeta(s)}.
\end{align*}

Therefore:
\begin{align*}
    \mathbb{P}\left(\frac{Y_n}{n} - \frac{Z}{n} > C\left(n, \delta\right)\right) \leq \frac{\delta}{n^s\zeta(s)}.
\end{align*}

Summing over all $n$:
\begin{align*}
    \sum_{n=1}^{\infty} \mathbb{P}\left(\frac{Y_n}{n} - \frac{Z}{n} > C\left(n, \delta\right)\right) & \leq \sum_{n=1}^{\infty} \frac{\delta}{n^s\zeta(s)}                 \\
                                                                                                     & = \delta \sum_{n=1}^{\infty} \frac{1}{n^s\zeta(s)}                  \\
                                                                                                     & = \delta \cdot \frac{1}{\zeta(s)} \sum_{n=1}^{\infty} \frac{1}{n^s} \\
                                                                                                     & = \delta \cdot \frac{\zeta(s)}{\zeta(s)}                            \\
                                                                                                     & = \delta.
\end{align*}

This demonstrates that the Gaussian noise implementation with the specified correction function satisfies condition \eqref{ass:correction}.

\section{Correctness of \dpsprt{} for Exponential Families: Proof of Theorem~\ref{thm:dpsprt-correctness}}
\label{appendix:correctness-proofs}

In this appendix, we provide the detailed proof of Theorem~\ref{thm:dpsprt-correctness}, which establishes the error guarantees of our \dpsprt{}. This result is central to our paper as it formally demonstrates that the algorithm achieves the desired privacy-utility trade-off while maintaining statistical validity.

\begin{proof}
    \textbf{Type I Error.} We prove the bound on type I error first. Let $\gamma \in (0,1)$ be an arbitrary parameter that allocates the error probability.
    \begin{align*}
        \mathbb{P}_{\theta_0} \big( \hat{d} = 1 \big) & \leq \mathbb{P}_{\theta_0} \big( \tau_1 < \infty \big)                                                                                                  \\
                                                      & = \mathbb{P}_{\theta_0}\bigg(\exists n \in \mathbb{N}_{>0}:\ \frac{\sum_{i=1}^n X_i}{n} - \mu_1 + \frac{Y_n}{n} \geq                                    \\
                                                      & \qquad\frac{-\KL(\nu_{\theta_1}, \nu_{\theta_0})+\log(1/(\gamma\alpha))/n}{\theta_1-\theta_0} + \frac{Z}{n} + C\left(n, (1-\gamma)\alpha\right) \bigg).
    \end{align*}

    By the union bound, we can separate this into two events:
    \begin{align*}
        \mathbb{P}_{\theta_0} \big( \hat{d} = 1 \big) & \leq \mathbb{P}_{\theta_0}\bigg(\exists n \in \mathbb{N}_{>0}:\ \frac{\sum_{i=1}^n X_i}{n} - \mu_1 \geq \frac{-\KL(\nu_{\theta_1}, \nu_{\theta_0})+\log(1/(\gamma\alpha))/n}{\theta_1-\theta_0}\bigg) \\
                                                      & \quad + \mathbb{P}_{\theta_0}\bigg(\exists n \in \mathbb{N}_{>0}:\ \frac{Y_n}{n} - \frac{Z_1}{n} \geq C\left(n, (1-\gamma)\alpha\right)\bigg).
    \end{align*}

    For the first term, we can recognize the probability of eventually rejecting $H_0$ in the non-private SPRT. Let us define $\tau_{\text{SPRT}}$ as the first time the likelihood ratio exceeds the threshold for rejecting $H_0$:
    \begin{align*}
        \tau_{\text{SPRT}} = \inf\bigg\{n \in \mathbb{N}_{>0} : \log\bigg(\prod_{i=1}^n \frac{f_{\theta_1}(X_i)}{f_{\theta_0}(X_i)}\bigg) \geq \log\bigg(\frac{1}{\gamma\alpha}\bigg)\bigg\}.
    \end{align*}

    For exponential families, this is equivalent to:
    \begin{align*}
        \tau_{\text{SPRT}} = \inf\bigg\{n \in \mathbb{N}_{>0} : \frac{\sum_{i=1}^n X_i}{n} - \mu_1 \geq \frac{-\KL(\nu_{\theta_1}, \nu_{\theta_0})+\log(1/(\gamma\alpha))/n}{\theta_1-\theta_0}\bigg\}.
    \end{align*}

    Thus, the first term equals $\mathbb{P}_{\theta_0}(\tau_{\text{SPRT}} < \infty)$, which is the probability under $\theta_0$ that we eventually reject $H_0$ in the non-private SPRT. Let $L_n = \prod_{i=1}^n \frac{f_{\theta_1}(X_i)}{f_{\theta_0}(X_i)}$ be the likelihood ratio.

    Applying Wald's likelihood ratio identity (i.e., a change-of-measure argument) to this specific stopping event, we get
    \begin{align*}
        \mathbb{P}_{\theta_0}(\tau_{\text{SPRT}} < \infty) & = \mathbb{E}_{\theta_1}\left[\frac{1}{L_{\tau_{\text{SPRT}}}} \cdot \mathbf{1}_{\{\tau_{\text{SPRT}} < \infty\}}\right].
    \end{align*}

    By definition of $\tau_{\text{SPRT}}$, on the event $\{\tau_{\text{SPRT}} < \infty\}$, we have $L_{\tau_{\text{SPRT}}} \geq \frac{1}{\gamma\alpha}$, which implies $\frac{1}{L_{\tau_{\text{SPRT}}}} \leq \gamma\alpha$. Therefore,
    \begin{align*}
        \mathbb{P}_{\theta_0}(\tau_{\text{SPRT}} < \infty) & = \mathbb{E}_{\theta_1}\left[\frac{1}{L_{\tau_{\text{SPRT}}}} \cdot \mathbf{1}_{\{\tau_{\text{SPRT}} < \infty\}}\right] \\
                                                           & \leq \mathbb{E}_{\theta_1}\left[\gamma\alpha \cdot \mathbf{1}_{\{\tau_{\text{SPRT}} < \infty\}}\right]                  \\
                                                           & = \gamma\alpha \cdot \mathbb{P}_{\theta_1}(\tau_{\text{SPRT}} < \infty)                                                 \\
                                                           & \leq \gamma\alpha.
    \end{align*}

    For the second term, using condition \eqref{ass:correction} and a union bound,
    \begin{align*}
        \mathbb{P}_{\theta_0}\bigg(\exists n \in \mathbb{N}_{>0}:\ \frac{Y_n}{n} - \frac{Z_1}{n} \geq C\left(n, (1-\gamma)\alpha\right)\bigg) & \leq \sum_{n=1}^{\infty}\mathbb{P}_{\theta_0}\bigg(\frac{Y_n}{n} - \frac{Z_1}{n} \geq C\left(n, (1-\gamma)\alpha\right)\bigg) \\
                                                                                                                                              & \leq (1-\gamma)\alpha.
    \end{align*}

    Therefore
    \begin{align*}
        \mathbb{P}_{\theta_0}\big( \hat{d} = 1 \big) & \leq \gamma\alpha + (1-\gamma)\alpha \\
                                                     & = \alpha.
    \end{align*}

    \textbf{Type II Error.} We now prove the bound on type II error. We have
    \begin{align*}
        \mathbb{P}_{\theta_1} ( \hat{d} = 0 ) & \leq \mathbb{P}_{\theta_1} ( \tau_0 < \infty )                                                                                                  \\
                                              & = \mathbb{P}_{\theta_1}\bigg(\exists n \in \mathbb{N}_{>0}: \frac{\sum_{i=1}^n X_i}{n} - \mu_0 + \frac{Y_n}{n} \leq                             \\
                                              & \qquad\qquad \frac{\KL(\nu_{\theta_0}, \nu_{\theta_1})-\log(1/(\gamma\beta))/n}{\theta_1-\theta_0} - \frac{Z}{n} - C(n, (1-\gamma)\beta) \bigg)
    \end{align*}

    By the union bound, we can separate this into two events:
    \begin{align*}
        \mathbb{P}_{\theta_1} ( \hat{d} = 0 ) & \leq \mathbb{P}_{\theta_1}\bigg(\exists n \in \mathbb{N}_{>0}: \frac{\sum_{i=1}^n X_i}{n} - \mu_0 \leq \frac{\KL(\nu_{\theta_0}, \nu_{\theta_1})-\log(1/(\gamma\beta))/n}{\theta_1-\theta_0}\bigg) \\
                                              & \quad + \mathbb{P}_{\theta_1}\bigg(\exists n \in \mathbb{N}_{>0}: \frac{Y_n}{n} + \frac{Z}{n} \leq -C(n, (1-\gamma)\beta)\bigg)
    \end{align*}

    For the first term, we recognize the probability of eventually accepting $H_0$ in the non-private SPRT when the truth is $H_1$. By a similar change-of-measure argument as for type I error, this probability is bounded by $\gamma\beta$.

    For the second term, using condition \eqref{ass:correction2}, we have:
    \begin{align*}
        \mathbb{P}_{\theta_1}\bigg(\exists n \in \mathbb{N}_{>0}: \frac{Y_n}{n} + \frac{Z}{n} \leq -C(n, (1-\gamma)\beta)\bigg) & \leq \sum_{n=1}^{\infty}\mathbb{P}_{\theta_1}\bigg(\frac{Y_n}{n} + \frac{Z}{n} < -C(n, (1-\gamma)\beta)\bigg) \\
                                                                                                                                & \leq (1-\gamma)\beta
    \end{align*}

    Therefore $\mathbb{P}_{\theta_1}( \hat{d} = 0 ) \leq \gamma\beta + (1-\gamma)\beta = \beta$.

    Finally, we note that if the distribution $\mathcal{D}_Y$ is symmetric around zero and condition \eqref{ass:correction} holds, then condition \eqref{ass:correction2} is automatically satisfied. This is because:
    \begin{align*}
        \sum_{n=1}^{\infty} \mathbb{P}\left(\frac{Y_n}{n} + \frac{Z}{n} < -C(n,\delta)\right) & = \sum_{n=1}^{\infty} \mathbb{P}\left(-\frac{Y_n}{n} - \frac{Z}{n} > C(n,\delta)\right)            \\
                                                                                              & = \sum_{n=1}^{\infty} \mathbb{P}\left(\frac{Y_n}{n} - \frac{Z}{n} > C(n,\delta)\right) \leq \delta
    \end{align*}
    where the second equality uses the symmetry of $Y_n$ around zero.
\end{proof}
\section{Sample Complexity of \dpsprt{} for Bernoulli: Proof of Theorem~\ref{thm:sc_general}}
\label{appendix:sample-complexity-proofs}

We begin with the proof of the generic sample complexity result (Theorem~\ref{thm:sc_general}), which provides an upper bound on the expected stopping time for any noise distribution satisfying appropriate conditions.

\begin{proof}[Proof of Theorem~\ref{thm:sc_general}]
    We first derive the upper bound on $\bE_{\theta_0}[\tau]$. Using the formula for the expectation of a non-negative integer-valued random variable, {\small
    \begin{align*}
        & \bE_{\theta_0}[\tau] \leq \bE_{\theta_0}[\tau_0] = 1 + \sum_{n=1}^{\infty} \bP_{\theta_0}(\tau_0 > n) \\
        &\leq 1 + \sum_{n=1}^{\infty} \bP_{\theta_0}\left(\frac{\sum_{i=1}^n X_i}{n} +\frac{Y_n}{n} > \mu_0 + \frac{\KL(\nu_{\theta_0}, \nu_{\theta_1})}{\theta_1-\theta_0} - \frac{\log(1/(\gamma\beta))/n}{\theta_1-\theta_0} + \frac{Z}{n} -C(n,(1-\gamma)\beta)\right) \\
        &\leq 1 + (1-\gamma)\beta + \sum_{n=1}^{\infty} \bP_{\theta_0}\left(\frac{\sum_{i=1}^n X_i}{n} > \mu_0 + \frac{\KL(\nu_{\theta_0}, \nu_{\theta_1})}{\theta_1-\theta_0} - \frac{\log(1/(\gamma\beta))/n}{\theta_1-\theta_0} - 2C(n,(1-\gamma)\beta)\right),
    \end{align*}}
    
   where in the last inequality we have used a union bound together with the assumption that $\sum_{n=1}^{\infty} \mathbb{P}\left(\frac{Y_n}{n} - \frac{Z}{n} > C(n,\delta)\right) \leq \delta$.

    Using the definition of $N(\theta_0,\theta_1,\beta,\gamma)$, which is the minimum number of samples needed to ensure that
    \[\frac{\log(1/(\gamma\beta))/n}{\theta_1-\theta_0} + 2C(n,(1-\gamma)\beta) \leq \frac{1}{2}\frac{\KL(\nu_{\theta_0}, \nu_{\theta_1})}{\theta_1 - \theta_0}\]
we can split the sum and write
    \begin{eqnarray*}
        \bE_{\theta_0}[\tau]  &\leq & 1 + (1-\gamma)\beta + N(\theta_0,\theta_1,\beta,\gamma) + \sum_{n=N_0}^{\infty} \bP_{\theta_0}\left(\frac{\sum_{i=1}^n X_i}{n}  >  \mu_0 + \frac{1}{2}\frac{\KL(\nu_{\theta_0}, \nu_{\theta_1})}{\theta_1 - \theta_0}\right) \\
        & \leq & 1 + (1-\gamma)\beta + N(\theta_0,\theta_1,\beta,\gamma) + \sum_{n=1}^{\infty} \exp\left(-2 \left(\frac{1}{2}\frac{\KL(\nu_{\theta_0}, \nu_{\theta_1})}{\theta_1 - \theta_0}\right)^2 n\right),
    \end{eqnarray*}
    where the last step uses Hoeffding's inequality, as Bernoulli distributions are bounded in $[0,1]$. Computing the series on the right-hand side and using Pinsker's inequality yields the bounded stated in Theorem~\ref{thm:sc_general}.
    
    We now upper bound  $\bE_{\theta_1}[\tau]$ following similar lines:
    \begin{align*}
        & \bE_{\theta_1}[\tau] \leq \bE_{\theta_1}[\tau_1] = 1 + \sum_{n=1}^{\infty} \bP_{\theta_1}(\tau_1 > n) \\
        &\leq 1 + \sum_{n=1}^{\infty} \bP_{\theta_1}\left(\frac{\sum_{i=1}^n X_i}{n} +\frac{Y_n}{n} < \mu_1 - \frac{\KL(\nu_{\theta_1}, \nu_{\theta_0})}{\theta_1-\theta_0} + \frac{\log(1/(\gamma\alpha))/n}{\theta_1-\theta_0} + \frac{Z}{n} +C(n,(1-\gamma)\alpha)\right) \\
                &\leq 1 + (1-\gamma)\alpha + \sum_{n=1}^{\infty} \bP_{\theta_1}\left(\frac{\sum_{i=1}^n X_i}{n} < \mu_1 - \frac{\KL(\nu_{\theta_1}, \nu_{\theta_0})}{\theta_1-\theta_0} + \frac{\log(1/(\gamma\alpha))/n}{\theta_1-\theta_0} +2C(n,(1-\gamma)\alpha)\right) \\
    \end{align*}
    \hspace{-0.2cm} where in the last inequality we have used that $\sum_{n=1}^{\infty} \mathbb{P}\left(\frac{Z}{n} - \frac{Y_n}{n} > C(n,\delta)\right) \leq \delta$.

    Using this time the definition of $N(\theta_1,\theta_0,\alpha,\gamma)$ yields 
        \begin{align*}
         \bE_{\theta_1}[\tau] & \leq 1 + (1-\gamma)\alpha + N(\theta_1,\theta_0,\alpha,\gamma) + \sum_{n=1}^{\infty} \bP_{\theta_1}\left(\frac{\sum_{i=1}^n X_i}{n} < \mu_1 - \frac{1}{2}\frac{\KL(\nu_{\theta_1}, \nu_{\theta_0})}{\theta_1-\theta_0}\right) \\
       &\leq 1 + (1-\gamma)\alpha + N(\theta_1,\theta_0,\alpha,\gamma) + \sum_{n=1}^{\infty} \exp\left(-2 \left(\frac{1}{2}\frac{\KL(\nu_{\theta_1}, \nu_{\theta_0})}{\theta_1 - \theta_0}\right)^2 n\right)
    \end{align*}
    The final bound on $\bE_{\theta_1}[\tau]$ follows by computing the last series and applying Pinsker's inequality.
    
\end{proof}

\subsection{Analysis for Laplace Noise}

For \dpsprt{} with Laplace noise, the correction function is:
\[C(n,\delta) = \frac{6\log(n^s\zeta(s)/2\delta)}{n\varepsilon}\;.\]

To upper bound $N(\theta_0,\theta_1,\beta,\gamma)$, we consider $n$ such that:
\[\frac{\log(1/(\gamma\beta))}{\theta_1-\theta_0}\frac{1}{n} + \frac{12\log(n^s\zeta(s)/2(1-\gamma)\beta)}{n\varepsilon} \leq \frac{1}{2}\frac{\KL(\nu_{\theta_0}, \nu_{\theta_1})}{\theta_1 - \theta_0},\]

or equivalently:
\[\frac{2\log(1/(\gamma\beta))}{\KL(\nu_{\theta_0}, \nu_{\theta_1})} + \frac{24(\theta_1 - \theta_0)\log(\zeta(s)/2(1-\gamma)\beta)}{\KL(\nu_{\theta_0}, \nu_{\theta_1})\varepsilon} + \frac{24s(\theta_1 - \theta_0)}{\KL(\nu_{\theta_0}, \nu_{\theta_1})\varepsilon}\log(n) \leq n.\]

Using Lemma \ref{lem:simplify-ineq} stated below, such an $n$, and therefore $N(\theta_0,\theta_1,\beta,\gamma)$, is upper bounded by:

\begin{eqnarray*}
    \frac{2\log(1/(\gamma\beta))}{\KL(\nu_{\theta_0}, \nu_{\theta_1})} &+ &\frac{24(\theta_1 - \theta_0)\log(\zeta(s)/2(1-\gamma)\beta)}{\KL(\nu_{\theta_0}, \nu_{\theta_1})\varepsilon} \\ &+& \frac{24s(\theta_1 - \theta_0)}{\KL(\nu_{\theta_0}, \nu_{\theta_1})\varepsilon}\log\left(\left[\frac{24s(\theta_1 - \theta_0)}{\KL(\nu_{\theta_0}, \nu_{\theta_1})\varepsilon}\right]^2 +  \frac{4\log(1/(\gamma\beta))}{\KL(\nu_{\theta_0}, \nu_{\theta_1})}\right.\\ && \qquad\qquad\qquad\qquad + \left.\frac{48(\theta_1 - \theta_0)\log(\zeta(s)/2(1-\gamma)\beta)}{\KL(\nu_{\theta_0}, \nu_{\theta_1})\varepsilon}\right).
\end{eqnarray*}

Hence it follows from Theorem~\ref{thm:sc_general} that 

\begin{eqnarray*}
   \bE_{\theta_0}[\tau] & \leq & \frac{2\log(1/(\gamma\beta))}{\KL(\nu_{\theta_0}, \nu_{\theta_1})} + \frac{24(\theta_1 - \theta_0)\log(\zeta(s)/2(1-\gamma)\beta)}{\KL(\nu_{\theta_0}, \nu_{\theta_1})\varepsilon} \\ &+& \frac{24s(\theta_1 - \theta_0)}{\KL(\nu_{\theta_0}, \nu_{\theta_1})\varepsilon}\log\left(\left[\frac{24s(\theta_1 - \theta_0)}{\KL(\nu_{\theta_0}, \nu_{\theta_1})\varepsilon}\right]^2 +  \frac{4\log(1/(\gamma\beta))}{\KL(\nu_{\theta_0}, \nu_{\theta_1})} + \frac{48(\theta_1 - \theta_0)\log(\frac{\zeta(s)}{2(1-\gamma)\beta})}{\KL(\nu_{\theta_0}, \nu_{\theta_1})\varepsilon}\right) \\ & +& 1 + (1-\gamma)\beta + \frac{1}{1 - \exp\left(-\frac{(\mu_1-\mu_0)^4}{2(\theta_1-\theta_0)^2}\right)}\;.
\end{eqnarray*}

A similar bound can be obtained for $\bE_{\theta_1}[\tau]$ by upper bounding $N(\theta_1,\theta_0,\alpha,\gamma)$. 

Looking at the leading terms in these bounds when $\alpha$ and $\beta$ go to zero yields 
\[\limsup_{\beta \rightarrow 0} \frac{\bE_{\theta_0}\left[\tau\right]}{\log(1/\beta)} \leq 48 \max\left(\frac{1}{\KL(\nu_{\theta_0}, \nu_{\theta_1})}, \frac{\theta_1 - \theta_0}{\KL(\nu_{\theta_0}, \nu_{\theta_1}) \varepsilon}\right)\]
and 
\[\limsup_{\alpha \rightarrow 0} \frac{\bE_{\theta_1}\left[\tau\right]}{\log(1/\alpha)} \leq 48 \max\left(\frac{1}{\KL(\nu_{\theta_1}, \nu_{\theta_0})}, \frac{\theta_1 - \theta_0}{\KL(\nu_{\theta_1}, \nu_{\theta_0}) \varepsilon}\right).\]

\begin{lemma}[Lemma 19 from \cite{Tirinzoni22EPRL}]\label{lem:simplify-ineq}
    Let $B,C\geq 1$. If $k \leq B\log(k)+C$, then
    \begin{align*}
        k
         & \leq B \log(B^2 + 2C)+C.
    \end{align*}
\end{lemma}


%
%

\subsection{Analysis for Gaussian Noise}

For \dpsprt{} with Gaussian noise, the correction function is:
\[C(n,\delta) = \frac{\sqrt{2(\sigma_Y^2 + \sigma_Z^2)\log(n^s\zeta(s)/2\delta)}}{n}\;.\]

To upper bound $N(\theta_0,\theta_1,\beta,\gamma)$, we need to find the minimum $n$ such that:
\[\frac{\log(1/(\gamma\beta))/n}{\theta_1-\theta_0} + 2C(n,(1-\gamma)\beta) \leq \frac{1}{2}\frac{\KL(\nu_{\theta_0}, \nu_{\theta_1})}{\theta_1 - \theta_0}.\]

It is sufficient to enforce that each term on the left-hand side is less than half of the right-hand side:

\begin{align*}
    \frac{\log(1/(\gamma\beta))/n}{\theta_1-\theta_0} & \leq \frac{1}{4}\frac{\KL(\nu_{\theta_0}, \nu_{\theta_1})}{\theta_1 - \theta_0}, \\
    2C(n,(1-\gamma)\beta)                             & \leq \frac{1}{4}\frac{\KL(\nu_{\theta_0}, \nu_{\theta_1})}{\theta_1 - \theta_0}.
\end{align*}

From the first condition, we directly get:
\[n \geq \frac{4\log(1/(\gamma\beta))}{\KL(\nu_{\theta_0}, \nu_{\theta_1})}.\]

For the second condition, substituting the correction function:
\[2 \cdot \frac{\sqrt{2(\sigma_Y^2 + \sigma_Z^2)\log(n^s\zeta(s)/2(1-\gamma)\beta)}}{n} \leq \frac{1}{4}\frac{\KL(\nu_{\theta_0}, \nu_{\theta_1})}{\theta_1 - \theta_0}.\]

Let $\mathfrak{C} = \frac{\KL(\nu_{\theta_0}, \nu_{\theta_1})}{\theta_1 - \theta_0}$. Squaring both sides and making the change of variable $m = n^2$:
\[8(\sigma_Y^2 + \sigma_Z^2)\log(m^{s/2}\zeta(s)/2(1-\gamma)\beta) \leq \frac{m \mathfrak{C}^2}{16}.\]

Rearranging:
\[m \geq \frac{128(\sigma_Y^2 + \sigma_Z^2)}{\mathfrak{C}^2}\log(m^{s/2}\zeta(s)/2(1-\gamma)\beta).\]

This is of the form $k \geq B\log(k) + C$ where:
\begin{align*}
    k & = m,                                                                                  \\
    B & = \frac{64s(\sigma_Y^2 + \sigma_Z^2)}{\mathfrak{C}^2},                                \\
    C & = \frac{128(\sigma_Y^2 + \sigma_Z^2)\log(\zeta(s)/2(1-\gamma)\beta)}{\mathfrak{C}^2}.
\end{align*}

Using Lemma \ref{lem:simplify-ineq}:
\[m \geq B\log(B^2 + 2C) + C.\]

Converting back to $n$ with $m = n^2$:
{\footnotesize
\begin{align*}
    n & \geq \frac{8\sqrt{(\sigma_Y^2 + \sigma_Z^2)}}{|\mathfrak{C}|}                                                                                                                                                             \\
      & \quad \cdot \sqrt{s\log\left(\frac{(64s)^2(\sigma_Y^2 + \sigma_Z^2)^2}{\mathfrak{C}^4} + \frac{256(\sigma_Y^2 + \sigma_Z^2)\log(\zeta(s)/2(1-\gamma)\beta)}{\mathfrak{C}^2}\right) + 2\log(\zeta(s)/2(1-\gamma)\beta)}\;.
\end{align*}}

Therefore, $N(\theta_0,\theta_1,\beta,\gamma)$ is upper bounded by the maximum of:
\[\frac{4\log(1/(\gamma\beta))}{\KL(\nu_{\theta_0}, \nu_{\theta_1})}\]
and
    {\footnotesize
        \begin{align*}
             & \frac{8\sqrt{(\sigma_Y^2 + \sigma_Z^2)}}{|\mathfrak{C}|}                                                                                                                                                                  \\
             & \quad \cdot \sqrt{s\log\left(\frac{(64s)^2(\sigma_Y^2 + \sigma_Z^2)^2}{\mathfrak{C}^4} + \frac{256(\sigma_Y^2 + \sigma_Z^2)\log(\zeta(s)/2(1-\gamma)\beta)}{\mathfrak{C}^2}\right) + 2\log(\zeta(s)/2(1-\gamma)\beta)}\;.
        \end{align*}}


\section{Lower Bound: Proof of Theorem \ref{thm:lower-bound}}
\label{appendix:lower-bound-proofs}

As the proof could generalize to other than Bernoulli distributions, we denote by $\nu_{\theta_0}$ (resp. $\nu_{\theta_1}$) the distribution of each $X_i$ under $\cH_0$ (resp. $\cH_1$) and by $\bP_{\theta_0}$ the probability space under which $(X_i)_{i\in \N}$ is an iid stream distribution under $\nu_{\theta_0}$ (resp. $\nu_{\theta_1}$).

We let $\bP_{\theta_i}^{Y}$ denote the distribution of a random vector $Y$ under the probability model $\bP_{\theta_i}$. To ease the notation, we will sometimes denote by $X_{1:t}$ the random vector $(X_1,\dots,X_t)$.

We establish the lower bound by combining information-theoretic inequalities with privacy constraints. Let $(\tau,d)$ be a $(\alpha,\beta)$-correct test. We assume that $\tau$ is almost surely finite under both $\bP_{\theta_0}$ and $\bP_{\theta_1}$, otherwise any lower bound on its expectation is trivially true.

\textbf{Step 1: Lower bounding a KL divergence for any $(\alpha,\beta)$-correct test.}
Using the data processing inequality, one can write
\begin{eqnarray*}
    \KL\left({\bP}_{\theta_0}^{(X_{1:\tau},\tau,\hat{d})},{\bP}_{\theta_1}^{(X_{1:\tau},\tau,\hat{d})}\right) &\geq& \KL\left({\bP}_{\theta_0}^{\ind(\hat{d} = 1)},{\bP}_{\theta_1}^{\ind(\hat{d} = 1)}\right) \\
    & = &\kl\left(\bP_{\theta_0}(\hat{d} = 1),\bP_{\theta_1}(\hat{d} = 1)\right)\\
    & \geq& \kl(\alpha, 1-\beta),
\end{eqnarray*}
where the last step uses the correctness of the test and monotonicity properties of the Bernoulli KL divergence, $\kl(x,y) = x\log(x/y) + (1-x)\log((1-x)/(1-y))$.

This argument is quite common to prove lower bounds on regret or sample complexity in the multi-armed bandit literature, and was popularized by, e.g. \cite{GMS16,GK16}.

\textbf{Step 2: Computing the KL for any test.} Still using standard arguments, we can establish the following equality:
\begin{equation}\KL\left({\bP}_{\theta_0}^{(X_{1:\tau},\tau,\hat{d})},{\bP}_{\theta_1}^{(X_{1:\tau},\tau,\hat{d})}\right) = \mathbb{E}_{\theta_0}[\tau] \cdot \KL(\nu_{\theta_0}, \nu_{\theta_1})\;.\label{ineq:all_tests}\end{equation}
For completeness, we provide the proof of this result.
Assuming $\nu_{\theta_i}$ has a density $f_{\theta_i}$ with respect to a common reference measure, we can express:
$$\KL\left({\bP}_{\theta_0}^{(X_{1:\tau},\tau,\hat{d})},{\bP}_{\theta_1}^{(X_{1:\tau},\tau,\hat{d})}\right) = \mathbb{E}_{\theta_0}\left[\log \prod_{i=1}^{\tau} \frac{f_{\theta_0}(X_i)}{f_{\theta_1}(X_i)}\right] = \bE_{\theta_0}\left[\sum_{i=1}^{\tau}\log \frac{f_{\theta_0}(X_i)}{f_{\theta_1}(X_i)}\right].$$
Let $Z_i = \log \frac{\mathbb{P}_{\theta_0}(X_i)}{\mathbb{P}_{\theta_1}(X_i)}$. Since the observations are i.i.d., each $Z_i$ has mean:
$\mathbb{E}_{\theta_0}[Z_i] = \KL(\nu_{\theta_0}, \nu_{\theta_1})$. By Wald's identity \cite{Wald45SPRT} (as $\tau$ has finite expectation), we have
$$\mathbb{E}_{\theta_0}\left[\sum_{i=1}^{\tau} Z_i\right] = \mathbb{E}_{\theta_0}[\tau] \cdot \mathbb{E}_{\theta_0}[Z_1] = \mathbb{E}_{\theta_0}[\tau] \cdot \KL(\mathbb{P}_{\theta_0}, \mathbb{P}_{\theta_1})\;,$$
and we obtain \eqref{ineq:all_tests}.

\textbf{Step 3: Upper Bounding the KL for Private Tests.} Using the privacy constraints, we now prove that for an $\varepsilon$-DP test we further have
\begin{equation}\KL\left({\bP}_{\theta_0}^{(X_{1:\tau},\tau,\hat{d})},{\bP}_{\theta_1}^{(X_{1:\tau},\tau,\hat{d})}\right) \leq \epsilon \cdot \TV(\nu_{\theta_0}, \nu_{\theta_1}) \cdot \mathbb{E}_{\theta_0}[\tau]\;.\label{ineq:using_privacy}\end{equation}
The argument is similar to the proof of Theorem 4.9 from \cite{azize2024privacy}, which applies to a more complex bandit model but with a finite horizon $T$. The idea is to build a coupling between the data collected under $\bP_{\theta_0}$ and $\bP_{\theta_1}$.

Specifically, we let $\underline{\nu}$ be the maximal coupling between $\nu_{\theta_0}$ and $\nu_{\theta_1}$: $(X,X') \sim \underline{\nu}$ is such that the marginal distribution of $X$ is $\nu_{\theta_0}$, the marginal distribution of $X'$ is $\nu_{\theta_1}$ and $\bE_{(X,X') \sim \underline{\nu}} [\ind(X \neq X')] = \TV(\nu_{\theta_0},\nu_{\theta_1})$. We now define the probability space $\bP_{\underline{\nu}}$ in which we collect iid pairs $(X_i,X'_i) \sim \underline{\nu}$ and let $(\tau,\hat{d})$ be the stopping and decision rules of the test based on the iid stream $X_i$ (distributed under $\nu_{\theta_0}$) and $(\tau',\hat{d}')$ be the stopping and decision rules of the test based on the iid stream $X'_i$ (distributed under $\nu_{\theta_1}$).

We observe that the marginal distribution of $(X_{1:\tau},\tau,\hat{d})$ in the probability space $\bP_{\underline{\nu}}$ is $\bP_{\theta_0}^{(X_{1:\tau}, \tau, \hat{d})}$ while that of $(X'_{1:\tau'},\tau',\hat{d}')$ is $\bP_{\theta_1}^{(X_{1:\tau}, \tau, \hat{d})}$. Using again the data-processing inequality,
\begin{eqnarray*}
    \KL\left({\bP}_{\theta_0}^{(X_{1:\tau},\tau,\hat{d})},{\bP}_{\theta_1}^{(X_{1:\tau},\tau,\hat{d})}\right) \leq  \KL\left({\bP}_{\underline{\nu}}^{(X_{1:\tau},X'_{1:\tau},\tau,\hat{d})},{\bP}_{\underline{\nu}}^{(X_{1:\tau'},X'_{1:\tau'},\tau',\hat{d}')}\right)\;.
\end{eqnarray*}


To compute this last KL-divergence, for clarity we will consider the special case of Bernoulli distributions that is the main focus of this paper. In that case the distribution of $(X_{1:\tau},X'_{1:\tau},\tau,\hat{d})$ and $(X_{1:\tau'},X'_{1:\tau'},\tau',\hat{d}')$ are discrete and we have
    {\footnotesize
        \begin{align*}
                 & \KL\left({\bP}_{\underline{\nu}}^{(X_{1:\tau},X'_{1:\tau},\tau,\hat{d})},{\bP}_{\underline{\nu}}^{(X_{1:\tau'},X'_{1:\tau'},\tau',\hat{d}')}\right) \\ = &  \sum_{\substack{T > 0, d \in \{0,1\} \\ x_{1:T} \in \{0,1\}^T \\ x'_{1:T} \in \{0,1\}^T}} \bP_{\underline{\nu}}\left(X_{1:T} = x_{1:T},X'_{1:T} = x'_{1:T}, \tau = T, \hat{d} = d \right) \log \frac{\bP_{\underline{\nu}}\left(X_{1:T} = x_{1:T},X'_{1:T} = x'_{1:T}, \tau = T, \hat{d} = d \right)}{\bP_{\underline{\nu}}\left(X_{1:T} = x_{1:T},X'_{1:T} = x'_{1:T}, \tau' = T, \hat{d}' = d \right)}\\
            \leq & \sum_{\substack{T > 0, d \in \{0,1\}                                                                                                                \\ x_{1:T} \in \{0,1\}^T \\ x'_{1:T} \in \{0,1\}^T}} \bP_{\underline{\nu}}\left(X_{1:T} = x_{1:T},X'_{1:T} = x'_{1:T}, \tau = T, \hat{d} = d \right) \log \frac{\bP_{\underline{\nu}}\left(\tau = T, \hat{d} = d | X_{1:T} = x_{1:T}\right)}{\bP_{\underline{\nu}}\left(\tau' = T, \hat{d}' = d | X'_{1:T} = x'_{1:T}\right)}.
        \end{align*}}
Using the fact that the test is $\varepsilon$-DP we obtain by the group-privacy property that
\[\frac{\bP_{\underline{\nu}}\left(\tau = T, \hat{d} = d | X_{1:T} = x_{1:T}\right)}{\bP_{\underline{\nu}}\left(\tau' = T, \hat{d}' = d | X'_{1:T} = x'_{1:T}\right)} \leq e^{\varepsilon \sum_{t=1}^{T} \ind(x_i \neq x'_i)},\]
hence
\begin{align*}
     & \KL\left({\bP}_{\underline{\nu}}^{(X_{1:\tau},X'_{1:\tau},\tau,\hat{d})},{\bP}_{\underline{\nu}}^{(X_{1:\tau'},X'_{1:\tau'},\tau',\hat{d}')}\right) \\  \leq &  \sum_{\substack{T > 0, d \in \{0,1\} \\ x_{1:T} \in \{0,1\}^T \\ x'_{1:T} \in \{0,1\}^T}} \bP_{\underline{\nu}}\left(X_{1:T} = x_{1:T},X'_{1:T} = x'_{1:T}, \tau = T, \hat{d} = d \right) \varepsilon \sum_{i=1}^{T} \ind(x_i \neq x'_i) \\
     & = \varepsilon\bE_{\underline{\nu}}\left[\sum_{t=1}^{\tau} \ind(X_i \neq X_i')\right]                                                                \\
     & = \varepsilon\bE_{\underline{\nu}}\left[\tau\right] \TV(\nu_{\theta_0},\nu_{\theta_1})                                                              \\
     & = \varepsilon\bE_{\theta_0}\left[\tau\right] \TV(\nu_{\theta_0},\nu_{\theta_1}),
\end{align*}
where the last but one step uses Wald's equality and the fact that $\bE_{\underline{\nu}}\left[ \ind(X_i \neq X_i')\right] = \TV(\nu_{\theta_0},\nu_{\theta_1})$ by definition of the coupling. This proves \eqref{ineq:using_privacy}.

\textbf{Step 4: Putting things together.}
Combining the lower bound with \eqref{ineq:all_tests} and the private upper bound \eqref{ineq:using_privacy}, we have
$$\kl(\alpha, 1-\beta) \leq \KL\left({\bP}_{\theta_0}^{(X_{1:\tau},\tau,\hat{d})},{\bP}_{\theta_1}^{(X_{1:\tau},\tau,\hat{d})}\right)\leq  \min\left(\mathbb{E}_{\theta_0}[\tau] \cdot \KL(\nu_{\theta_0}, \nu_{\theta_1}), \epsilon \cdot \TV(\nu_{\theta_0}, \nu_{\theta_1}) \cdot \mathbb{E}_{\theta_0}[\tau]\right).$$
Rearranging, we obtain:
$$\mathbb{E}_{\theta_0}[\tau] \geq \frac{\kl(\alpha, 1-\beta)}{\min\left(\KL(\nu_{\theta_0}, \nu_{\theta_1}), \epsilon \cdot \TV(\nu_{\theta_0}, \nu_{\theta_1})\right)}\;.$$
The lower bound on $\bE_{\theta_1}[\tau]$ follows in the same way by swapping the roles of $\theta_0$ and $\theta_1$.

\section{Enhancing \dpsprt{} with Subsampling}
\label{appendix:subsampling}

Privacy amplification by subsampling \citep{steinke2022composition} provides an approach to enhance the privacy-utility trade-off in adifferentially private algorithms. When a dataset is randomly subsampled before applying a private mechanism, the effective privacy cost can be reduced roughly proportionally to the sampling rate for pure differentially private mechanisms.

In our DP-SPRT context, subsampling allows us to improve privacy for a given noise level, or equivalently, to maintain the same level of privacy while requiring less noise. At each step when a new observation $X_i$ arrives, we implement Bernoulli subsampling by including it in our computation with probability $r \in (0,1)$, independent of all previous decisions. We maintain a running sum $S_n$ and count $M_n$ of the included observations:

\begin{enumerate}
    \item For each arriving observation $X_i$, include it with probability $r$ (draw $B_i \sim \text{Bernoulli}(r)$ independently);
    \item Update $S_n = \sum_{i=1}^n B_i X_i$ and $M_n = \sum_{i=1}^n B_i$;
    \item Compute the subsampled empirical mean $\bar{X}^r_n = S_n / M_n$ when $M_n > 0$.
\end{enumerate}

We chose this Bernoulli subsampling approach over alternatives like sampling without replacement because it requires only constant space and time complexity at each time step and thus can directly handle streaming observations with limited memory, providing significant computational advantages for long trials or challenging problem instances.

We incorporate subsampling into our \dpsprt{} with Laplace noise by using noise scaled for $\epsilon/r$-DP, since subsampling with rate $r$ amplifies privacy by this factor, resulting in $\epsilon$-DP overall. The subsampling mechanism is defined as follows:

\begin{align*}
    \tau^{\epsilon,r}   & = \min(\tau_0^{\epsilon,r}, \tau_1^{\epsilon,r})                                                                                                                       \\[0.5em]
    \tau_0^{\epsilon,r} & = \inf\bigg\{n \in \mathbb{N}_{>0}:
    \textcolor{red}{\bar{X}^r_n} + \textcolor{red}{r} \frac{Y_n}{n} \leq \mu_0 + \frac{\KL(\nu_{\theta_0}, \nu_{\theta_1})-\log(1/(\gamma\beta))/\textcolor{red}{M_n}}{\theta_1-\theta_0}        \\
                        & \qquad\qquad\qquad\qquad\qquad\qquad\qquad\qquad\, - \textcolor{red}{r} \frac{Z}{n} - \textcolor{red}{r}\frac{6\log(n^s\zeta(s)/(1-\gamma)\beta)}{n\epsilon} \bigg\}   \\[0.5em]
    \tau_1^{\epsilon,r} & = \inf\bigg\{n \in \mathbb{N}_{>0}:
    \textcolor{red}{\bar{X}^r_n} + \textcolor{red}{r} \frac{Y_n}{n} \geq \mu_1 + \frac{-\KL(\nu_{\theta_1}, \nu_{\theta_0})+\log(1/(\gamma\alpha))/\textcolor{red}{M_n}}{\theta_1-\theta_0}      \\
                        & \qquad\qquad\qquad\qquad\qquad\qquad\qquad\qquad\, + \textcolor{red}{r} \frac{Z}{n} + \textcolor{red}{r}\frac{6\log(n^s\zeta(s)/(1-\gamma)\alpha)}{n\epsilon} \bigg\},
\end{align*}

where $M_n = \sum_{i=1}^n B_i$ is the actual count of included observations, $\textcolor{red}{\bar{X}^r_n = S_n/M_n}$ is the subsampled empirical mean, and $Y_n \sim \text{Lap}(4/\epsilon)$, $Z \sim \text{Lap}(2/\epsilon)$ as in standard DP-SPRT. The decision rule remains $\hat{d}^{\epsilon,r} = i$ if $\tau^{\epsilon,r} = \tau_i^{\epsilon,r}$. The terms in \textcolor{red}{red} are the modifications compared to the standard \dpsprt{} algorithm.

\textbf{Privacy and Correctness:} The subsampled DP-SPRT maintains $\epsilon$-differential privacy and satisfies the error bounds $\mathbb{P}_{\theta_0}(\hat{d}^{\epsilon,r} = 1) \leq \alpha$ and $\mathbb{P}_{\theta_1}(\hat{d}^{\epsilon,r} = 0) \leq \beta$. The privacy follows from the subsampling amplification applied to the standard DP-SPRT analysis, while correctness uses appropriately scaled thresholds.

\textbf{Sample Complexity:} For the asymptotic regime where $\alpha, \beta \to 0$, the expected stopping time will be of the form:
\begin{align*}
    \limsup_{\beta \rightarrow 0} \frac{\mathbb{E}_{\theta_0}[\tau_0^{\epsilon,r}]}{\log(1/\beta)} \leq 48\max\left(\frac{1}{r \cdot \KL(\nu_{\theta_0}, \nu_{\theta_1})}, \frac{r \cdot (\theta_1-\theta_0)}{\KL(\nu_{\theta_0}, \nu_{\theta_1}) \epsilon}\right).
\end{align*}

This bound reveals a fundamental trade-off: smaller subsampling rates $r$ lead to worse statistical efficiency—the first term $1/(r \cdot \KL)$ increases due to higher variance of the subsampled empirical mean—but provide better privacy amplification as the second term $r \cdot (\theta_1-\theta_0)/(\KL \cdot \epsilon)$ decreases, reducing the impact of privacy noise. The optimal sampling rate depends on which effect dominates, determined by the privacy parameter $\epsilon$ and how hard the hypotheses are to distinguish. When $\epsilon$ is small and high privacy is required, aggressive subsampling with small $r$ can be beneficial as privacy amplification becomes more valuable. When $\epsilon$ is large and privacy is relaxed, less aggressive subsampling is preferable to maintain statistical efficiency. As demonstrated in our experiments (Section~\ref{sec:experiments}), this trade-off is particularly pronounced in high-privacy regimes.

For practical applications, the bound suggests setting the sampling rate as $r \approx \min(1,\sqrt{\epsilon/(\theta_1-\theta_0)})$ to balance the statistical and privacy terms. However, in our experiments we used the arbitrary choice $r = \min(1, \sqrt{\epsilon/10})$, which already worked very well in practice for the range of parameter values we considered.

\section{Complementary Experimental Analysis}
\label{appendix:complementary-experiments}

This section extends our experimental analysis beyond Section~\ref{sec:experiments}, examining: (1) performance across problem instances of varying difficulty, (2) stopping time distributions, (3) effects of target error probabilities and privacy parameters, and (4) a tuned version of our \dpsprt{} with improved sample complexity.

We evaluate our methods across three problem instances of varying difficulty:
\begin{itemize}
    \item \textbf{Instance 1 (Standard)}: $p_0 = 0.3$, $p_1 = 0.7$ - Well-separated parameters (as in the main paper)
    \item \textbf{Instance 2 (Close)}: $p_0 = 0.45$, $p_1 = 0.55$ - Parameters close to each other, requiring more samples
    \item \textbf{Instance 3 (Boundary)}: $p_0 = 0.05$, $p_1 = 0.25$ - One parameter near boundary, affecting observation variance
\end{itemize}

\subsection{Comparative Performance Across Problem Difficulty}

Figure~\ref{fig:instance_comparison} presents the sample complexity across the three problem instances.

\begin{figure}[ht]
    \centering
    \includegraphics[width=0.8\textwidth]{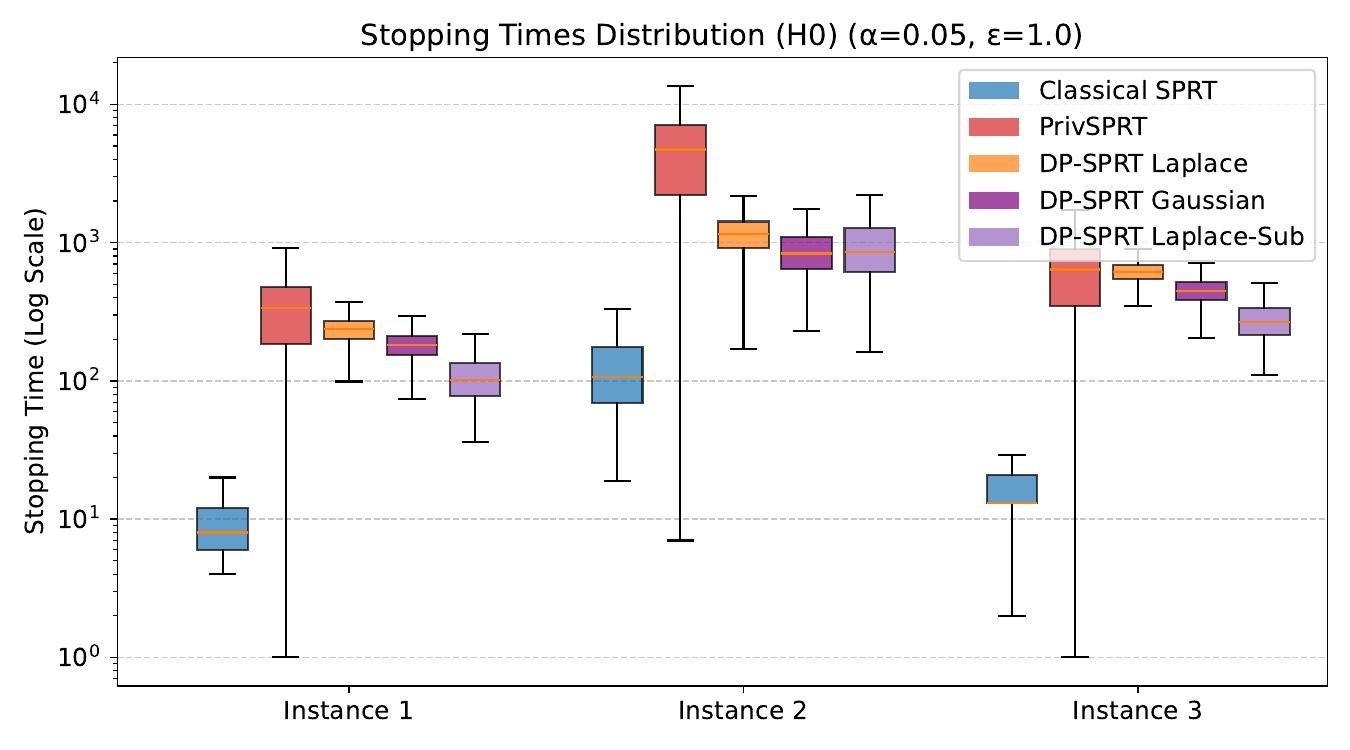}
    \includegraphics[width=0.8\textwidth]{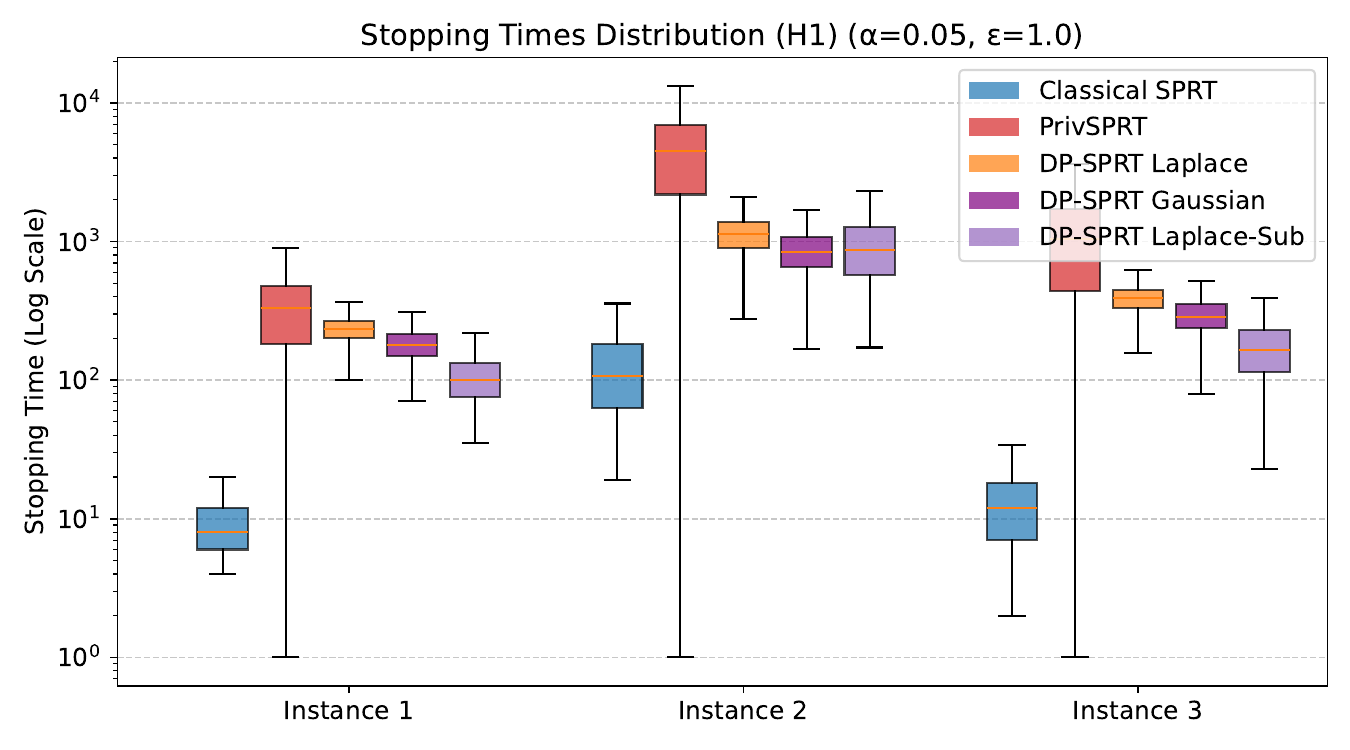}
    \caption{Sample complexity for different methods across problem instances with $\alpha = \beta = 0.05$ and $\epsilon = 1$. Top: Under $\mathcal{H}_0$. Bottom: Under $\mathcal{H}_1$.}
    \label{fig:instance_comparison}
\end{figure}

As difficulty increases from Instance 1 to Instance 2, all methods require more samples. The advantage of our \dpsprt{} variants over PrivSPRT becomes more pronounced in more challenging instances, with subsampling showing particular benefits when parameters are close (Instance 2).

Instance 3 exhibits asymmetry due to $p_0$ being close to zero, with decisions under $\mathcal{H}_0$ requiring fewer samples than under $\mathcal{H}_1$ for PrivSPRT. For \dpsprt, we observe the opposite effect, although it is to a lesser extent. Our methods maintain their advantage in both scenarios, more prominently under $\mathcal{H}_1$.

\subsection{Effect of Target Error Probability}

Figure~\ref{fig:alpha_comparison} shows how sample complexity varies with different target error probabilities.

\begin{figure}[ht]
    \centering
    \includegraphics[width=0.8\textwidth]{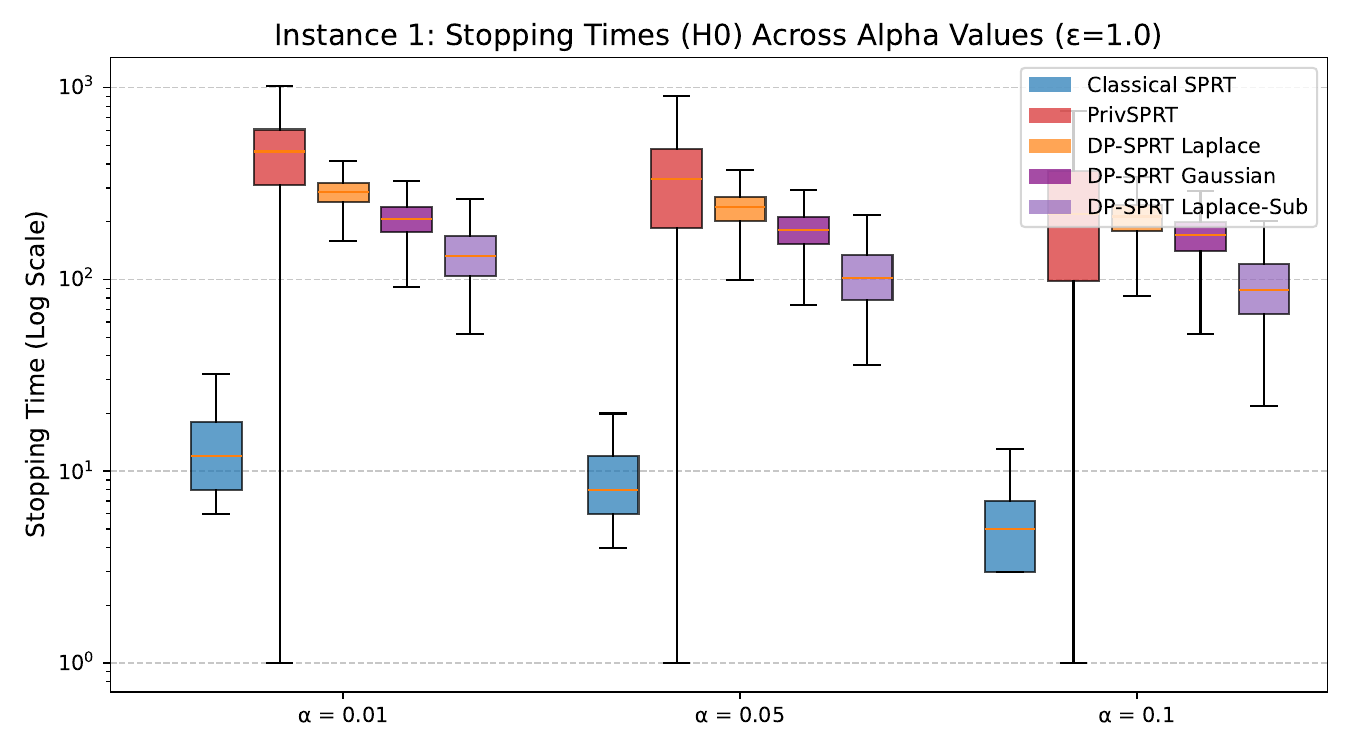}
    \includegraphics[width=0.8\textwidth]{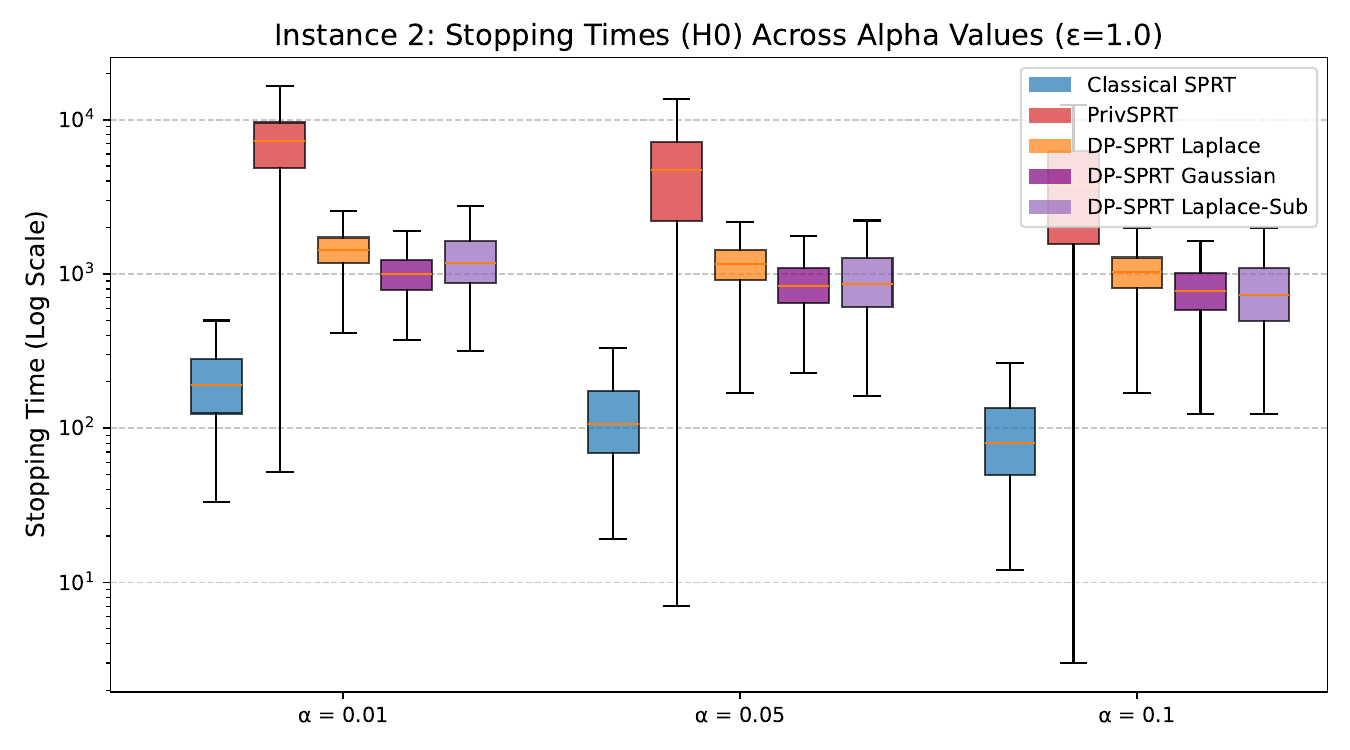}
    \caption{Sample complexity for different methods with varying target error probabilities at $\epsilon = 1$ and $\alpha = \beta \in \{0.01, 0.05, 0.1\}$. Top: Instance 1 under $\mathcal{H}_0$. Bottom: Instance 2 under $\mathcal{H}_0$.}
    \label{fig:alpha_comparison}
\end{figure}

Stricter error requirements (smaller $\alpha$ and $\beta$) predictably increase sample complexity for all methods. At $\alpha = \beta = 0.01$, we observe a wider gap between private and non-private methods.

\subsection{Privacy-Sample Complexity Trade-off}

Figure~\ref{fig:epsilon_comparison} demonstrates the relationship between privacy parameters and sample complexity.

\begin{figure}[ht]
    \centering
    \includegraphics[width=0.8\textwidth]{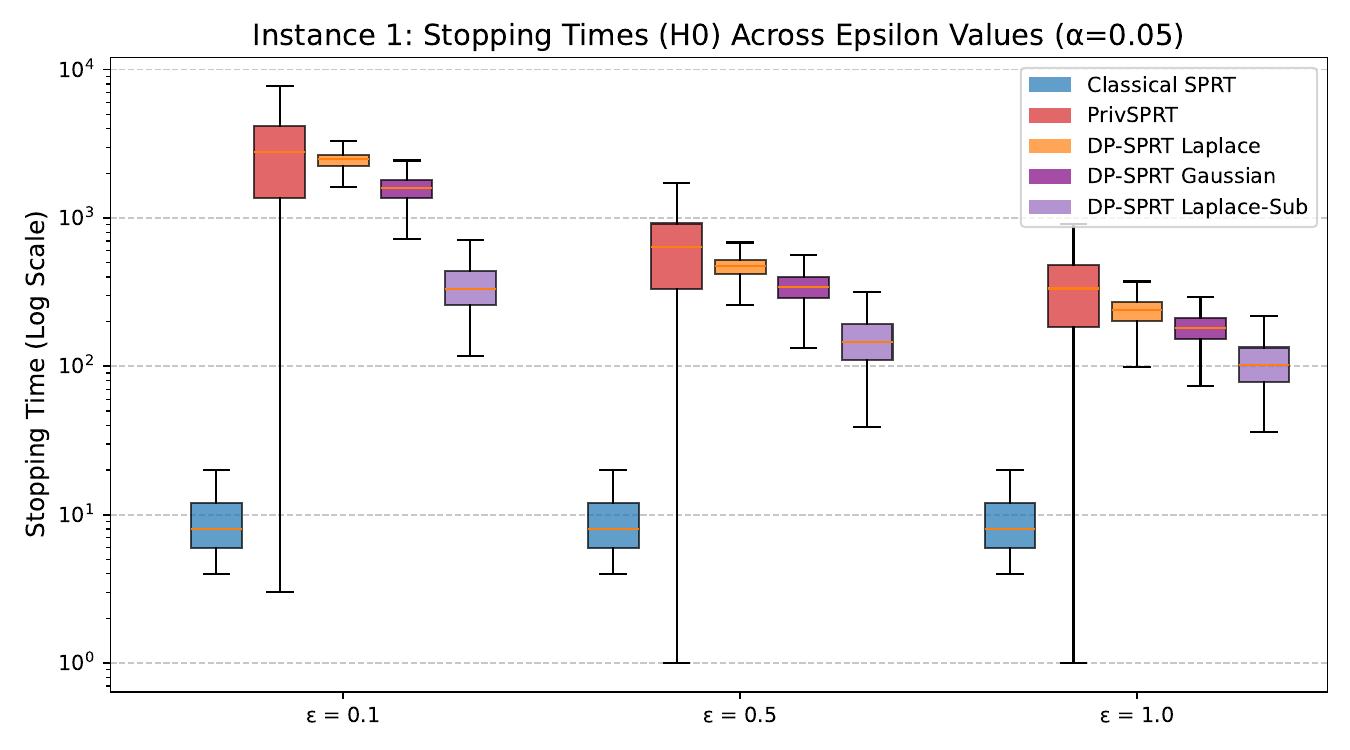}
    \includegraphics[width=0.8\textwidth]{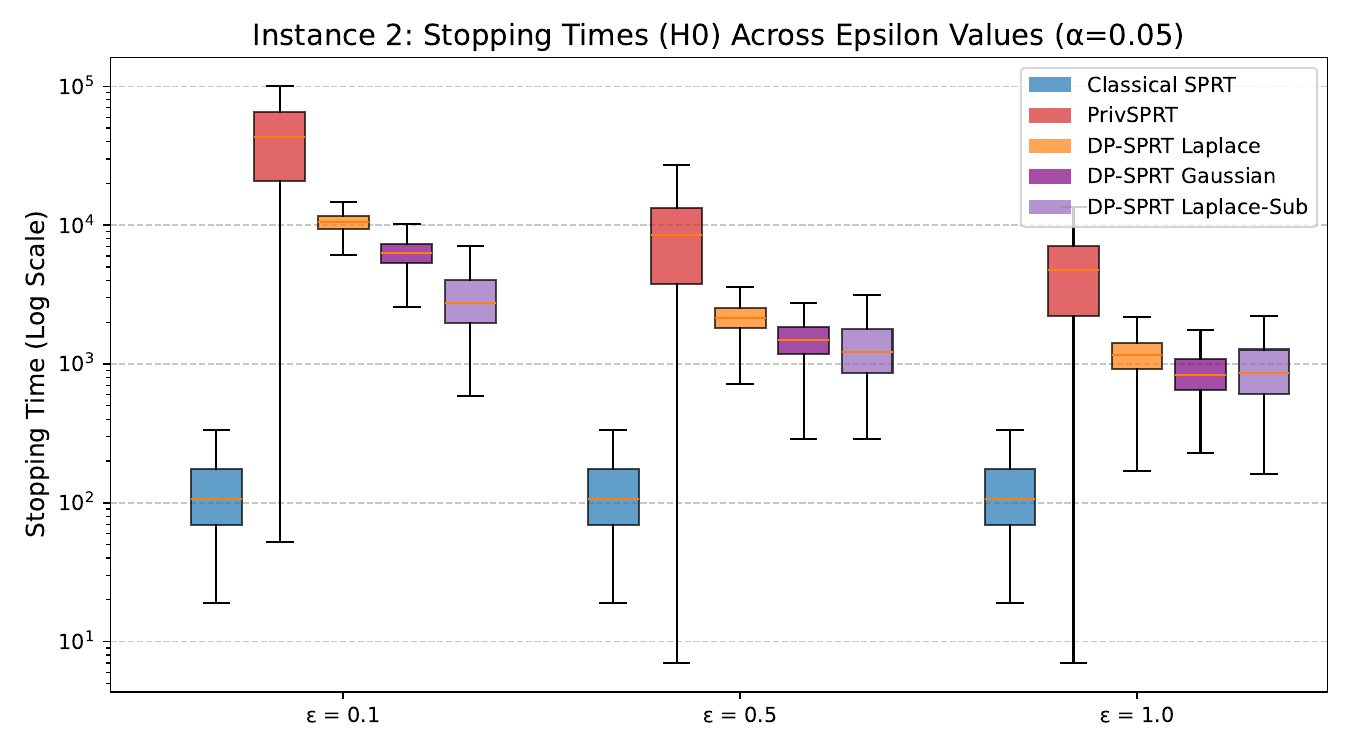}
    \caption{Sample complexity for different privacy parameters with $\alpha = \beta = 0.05$ and $\epsilon \in \{0.1, 1, 5\}$. Top: Instance 1 under $\mathcal{H}_0$. Bottom: Instance 2 under $\mathcal{H}_0$.}
    \label{fig:epsilon_comparison}
\end{figure}

The privacy-utility trade-off is evident as $\epsilon$ varies from 0.1 (high privacy) to 5 (relaxed privacy). At $\epsilon = 0.1$, all private methods require substantially more samples than non-private SPRT, with \dpsprt{} with subsampling maintaining a clear advantage. As privacy constraints relax with increasing $\epsilon$, all \dpsprt{} variants converge toward classical SPRT performance.

PrivSPRT exhibits higher variance in stopping time, particularly at lower $\epsilon$ values. Our \dpsprt{} variants show more consistent performance across trials, beneficial for applications requiring predictable resource allocation.

\subsection{Tuned Correction Function and Distribution of Stopping Times}

Our theoretical correction function $C(n,\delta)$ guarantees desired error probabilities but may be conservative. We explored a tuned version by scaling the function by the factor that maintains error probability targets, estimated through simulation.

The tuning process involves finding a scaling factor $\kappa \in (0,1)$ to replace $C(n,\delta)$ with $\kappa \cdot C(n,\delta)$. For Instance 1 with $\epsilon = 1$ and $\alpha = \beta = 0.1$, we found $\kappa \approx 0.5$ maintains error probabilities below target. However, like PrivSPRT's threshold, this value must be recomputed for new parameters.

\begin{figure}[ht]
    \centering
    \includegraphics[width=\textwidth]{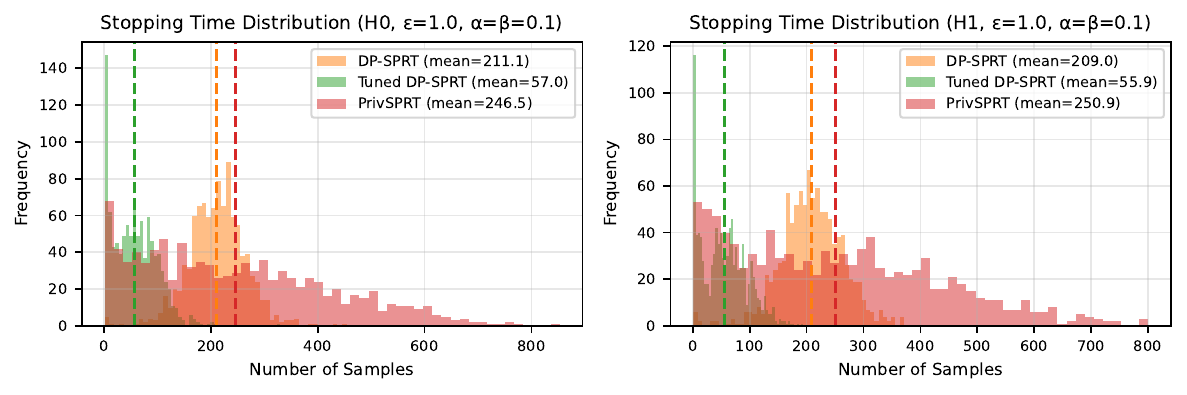}
    \includegraphics[width=\textwidth]{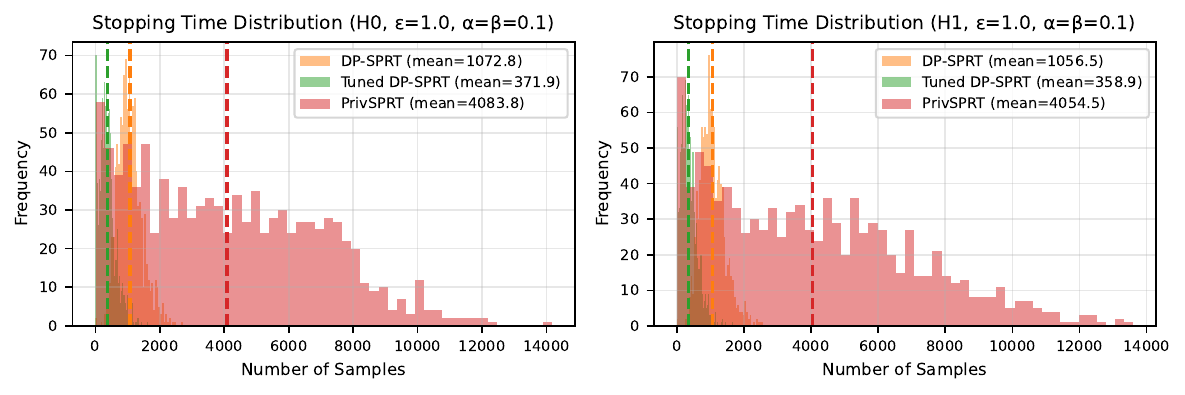}
    \includegraphics[width=\textwidth]{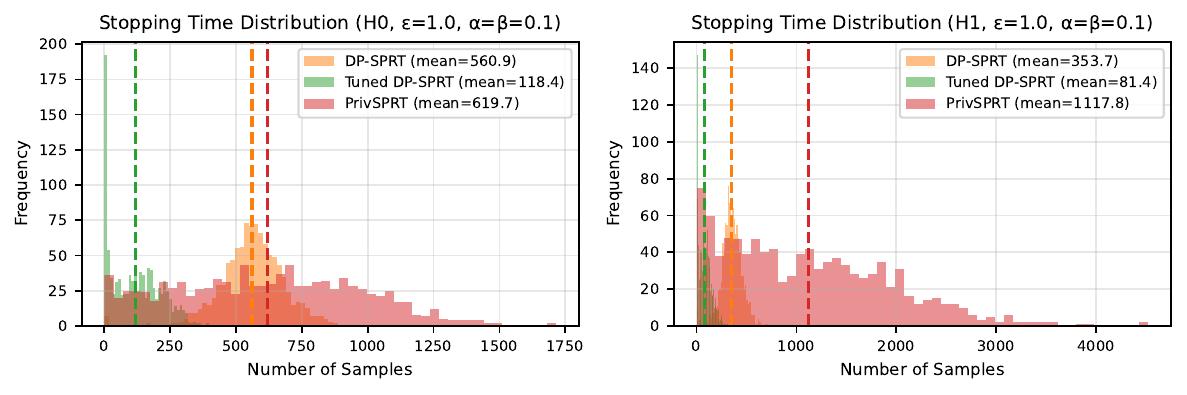}
    \caption{Histogram of stopping times for \dpsprt{} with Laplace noise, tuned \dpsprt{}, and PrivSPRT with $\alpha = \beta = 0.1$ and $\epsilon = 1$. Each row shows results under $\mathcal{H}_0$ (left) and $\mathcal{H}_1$ (right) for different instances. Top row: Instance 1. Middle row: Instance 2. Bottom row: Instance 3. Dashed lines indicate mean stopping times.}
    \label{fig:stopping_distribution}
\end{figure}

Figure~\ref{fig:stopping_distribution} compares stopping time distributions.  The tuned version of \dpsprt{} significantly reduces the number of samples compared to our standard version, however, it does not have the same theoretical backing. PrivSPRT exhibits both a larger mean and more spread out distribution, indicating more frequent occurrences of very long stopping times.

This visualization highlights not just average performance differences but also consistency and predictability. The lower variance of our methods represents a practical advantage for resource planning. The tuning approach offers a middle ground, leveraging theoretical guarantees while improving empirical performance, though like PrivSPRT, it requires empirical tuning for each specific setting.

\subsection{Summary of Findings}

Our complementary experiments reveal several important insights:

\begin{enumerate}
    \item The advantage of our \dpsprt{} methods becomes more pronounced as problem difficulty increases, making them particularly valuable for challenging testing scenarios with parameters that are close to each other.

    \item All \dpsprt{} variants exhibit lower variance in stopping times compared to PrivSPRT, leading to more predictable performance.

    \item In asymmetric scenarios (e.g., Instance 3 with parameters near boundaries), decisions under different hypotheses require substantially different sample sizes for PrivSPRT. Our method is less sensitive to this asymmetry.

    \item The subsampling approach provides substantial benefits in high-privacy regimes ($\epsilon = 0.1$).

    \item Our theoretical correction function is conservative, and a tuned version scaling the function by a factor $\kappa \approx 0.5$ can significantly improve sample complexity while maintaining error guarantees.

\end{enumerate}


\end{document}